\documentclass[10pt]{article} 
\usepackage[accepted]{tmlr}

\usepackage{hyperref}
\usepackage{url}

\usepackage{graphicx}
\usepackage{booktabs}
\usepackage{algorithm}
\usepackage{algpseudocode}
\usepackage{wrapfig}
\usepackage{color, colortbl}
\definecolor{Gray}{gray}{0.92}
\usepackage{tabularx,arydshln}
\usepackage{tablefootnote}

\usepackage{graphicx,amsfonts,amscd,amssymb,bm,url,color,latexsym,bbm,amsthm,amsmath}
\usepackage{physics}
\allowdisplaybreaks
\usepackage[capitalize,nameinlink]{cleveref}
\usepackage{hyperref}

\newtheorem{theorem}{Theorem}[section]
\newtheorem{lemma}[theorem]{Lemma}

\renewcommand{\mathbf}{\boldsymbol}

\newcommand{\mb}{\mathbf}
\newcommand{\mc}{\mathcal}

\newcommand{\bb}{\mathbb}

\newcommand{\set}[1]{\left\{ #1 \right\}}

\newcommand{\reals}{\bb R}

\newcommand{\eps}{\varepsilon}
\newcommand{\R}{\reals}

\newcommand{\paren}{\pqty}
\newcommand{\brac}{\bqty}

\DeclareMathOperator{\sign}{sign}

\DeclareMathOperator{\st}{s.t.}

\DeclareMathOperator*{\argmax}{arg\,max}
\DeclareMathOperator*{\argmin}{arg\,min}

\newcommand{\wh}{\widehat}

\newcommand{\T}{\intercal}

\usepackage{amsmath,amsfonts,bm}
\usepackage{ amssymb }

\def\eqref#1{equation~\ref{#1}}

\def\1{\bm{1}}

\DeclareMathAlphabet{\mathsfit}{\encodingdefault}{\sfdefault}{m}{sl}
\SetMathAlphabet{\mathsfit}{bold}{\encodingdefault}{\sfdefault}{bx}{n}

\def\gD{{\mathcal{D}}}

\def\gX{{\mathcal{X}}}
\def\gY{{\mathcal{Y}}}

\usepackage{mathtools}

\usepackage{enumitem}

\newcommand{\hy}[1]{\textcolor{black}{#1}}

\title{Selective Classification Under Distribution Shifts}

\author{\name Hengyue Liang \email liang656@umn.edu \\
      \addr Department of Electrical and Computer Engineering\\
            University of Minnesota
      \AND
      \name Le Peng \email peng0347@umn.edu \\
      \addr Department of Computer Science and Engineering \\
            University of Minnesota
      \AND
      \name Ju Sun \email jusun@umn.edu\\
      \addr Department of Computer Science and Engineering \\
            University of Minnesota}

\def\month{10}  

\def\year{2024} 
\def\openreview{\url{https://openreview.net/forum?id=dmxMGW6J7N}} 

\begin{document}

\maketitle

\begin{abstract}
In selective classification (SC), a classifier abstains from making predictions that are likely to be wrong to avoid excessive errors. To deploy imperfect classifiers---either due to intrinsic statistical noise of data or for robustness issue of the classifier or beyond---in high-stakes scenarios, SC appears to be an attractive and necessary path to follow. Despite decades of research in SC, most previous SC methods still focus on the ideal statistical setting only, i.e., the data distribution at deployment is the same as that of training, although practical data can come from the wild. To bridge this gap, in this paper, we propose an SC framework that takes into account distribution shifts, termed \emph{generalized selective classification}, that covers label-shifted (or out-of-distribution) and covariate-shifted samples, in addition to typical in-distribution samples, \emph{the first of its kind} in the SC literature. We focus on non-training-based confidence-score functions for generalized SC on deep learning (DL) classifiers, and propose two novel margin-based score functions. Through extensive analysis and experiments, we show that our proposed score functions are more effective and reliable than the existing ones for generalized SC on a variety of classification tasks and DL classifiers. The code is available at \url{https://github.com/sun-umn/sc_with_distshift}.  

\end{abstract}

\vspace{-1em}
\section{Introduction}
\label{Sec: introduction} 
In practice, classifiers almost never have perfect accuracy. Although modern classifiers powered by deep neural networks (DNNs) typically achieve higher accuracy than the classical ones, they are known to be unrobust: perturbations of inputs that are inconsequential to human decision making can easily alter DNN classifiers' predictions~\citep{carlini2019evaluating, croce2020robustbench,hendrycks2018benchmarking,liang2023optimization}, and more generally, shifts in data distribution in deployment from that in training often cause systematic classification errors. These classification errors, regardless of their source, are rarely acceptable for high-stakes applications, such as disease diagnosis in healthcare. 

To achieve minimal and controllable levels of classification error so that imperfect and unrobust classifiers can be deployed for high-stakes applications, a promising approach is \emph{selective classification} (SC): samples that are likely to be misclassified are selected, excluded from prediction, and deferred to human decision makers, so that the classification performance on the remaining samples reaches the desired level~\citep{chow1970optimum,franc2023optimal,geifman2017selective}. For example, by flagging and passing uncertain patient cases that it tends to mistake on to human doctors, an intelligent medical agent can make confident and correct diagnoses for the rest. This ``conservative'' classification framework not only saves doctors' efforts, but also avoids liability due to the agent's mistakes. 

Consider a multiclass classification problem with input space $\gX \in \R^n$, label space $\gY = \set{1, \dots, K}$, and training distribution $\mc D_{\mc X, \mc Y}$ on $\mc X \times \mc Y$. For any classifier $f:\mc X \to \mc Y$, there are many potential causes of classification errors. In this paper, we focus on three types of errors that are commonly encountered in practice and are studied extensively, but mostly separately, in the literature. 
\begin{itemize}[leftmargin=1em,nosep]
    \item \texttt{Type A errors}: errors made on \emph{in-distribution} (In-D) samples, i.e., those samples drawn from $\mc D_{\mc X, \mc Y}$. These are classification errors discussed in typical statistical learning frameworks~\citep{mohri2018foundations}; 
    
    \item \texttt{Type B errors}: errors made on \emph{label-shifted} samples, i.e., those samples with groundtruth labels not from $\gY$. Since $f$ assigns labels from $\gY$ only, it always errs on these samples; 
    
    \item \texttt{Type C errors}: errors made on \emph{covariate-shifted} samples, i.e., samples drawn from a different input distribution $\mc D'_{\mc X}$ where $\mc D'_{\mc X} \neq \mc D_{\mc X}$ but with groundtruth labels from $\mc Y$.  
\end{itemize} 
It is clear that in practical deployment of classifiers, samples can come from the wild, and hence \texttt{Type A}, \texttt{Type B} and \texttt{Type C} errors can coexist. In order to ensure the reliable deployment of classifiers in high-stakes applications, we must control the three types of errors, \emph{jointly}. Unfortunately, previous research falls short of a unified treatment of these errors. Classical SC \citep{chow1970optimum} focuses on rejecting samples that cause In-D errors \texttt{(Type A)}, whereas the current \emph{out-of-distribution (OOD) detection} research \citep{yang2021generalized,park2023nearest} focuses on detecting label-shifted samples \texttt{(Type B)}. Although \citet{hendrycks2016baseline,granese2021doctor,xia2022augmenting,kim2023unified} have advocated the simultaneous detection of samples that cause \texttt{Type A} and \texttt{Type B} errors, their approaches still treat the problem as consisting of two \emph{separate} tasks, reflected in their \emph{separate and independent} performance evaluation on OOD detection and SC. Regarding the challenge posed by \texttt{Type C} errors, existing work \citep{hendrycks2018benchmarking, croce2020robustbench} focuses primarily on obtaining classifiers that are more robust to covariate shifts, not on rejecting potentially misclassified samples due to covariate shifts---the latter, to the best of our knowledge, has not yet been explicitly considered, not to mention joint rejection together with \texttt{Type A} and \texttt{Type B} errors. 

In this paper, our goal is to close the gap and consider, \emph{for the first time}, rejecting all three types of errors in a unified framework. For brevity, we use the umbrella term \emph{distribution shifts} to cover both label shifts and covariate shifts, which are perhaps the most commonly seen types of distribution shifts, with the caveat that practical distribution shifts can also be induced by other sources. So, we call the unified framework considered in this paper \emph{selective classification under distribution shifts}, or \emph{generalized selective classification}. Another key desideratum is practicality.  With the increasing popularity of foundation models and associated downstream few-shot learners~\citep{brown2020language,radford2021learning,yuan2021florence}, accessing massive original training data becomes increasingly more difficult. Moreover, there are numerous high-stakes domains where training data are typically protected due to privacy concerns, such as healthcare and finance. These applied scenarios call for \emph{SC strategies that can work with given pretrained classifiers and do not require access to the training data}, which will be our focus in this paper. \textbf{Our contributions} include: 
\begin{itemize}[leftmargin=1em,nosep]
    \item We \hy{advocate} a new SC framework, \emph{generalized selective classification}, which rejects samples that could cause \texttt{Type A}, \texttt{Type B} and \texttt{Type C} errors \emph{jointly}, to improve classification performance over the non-rejected samples. With careful review and reasoning, we argue that generalized SC covers and unifies the scope of the existing OOD detection and SC, if \emph{the goal is to achieve reliable classification on the selected samples}. (\cref{sec:sc_dist_shift,Subsec: eva SC OOD})
    
    \item Focused on non-training-based (or post-hoc) SC settings, we identify a critical scale-sensitivity issue of several SC confidence scores based on softmax responses (\cref{subsec: SR-based score is problematic}) which are popularly used and reported to be the state-of-the-art (SOTA) methods in the existing SC literature \citep{geifman2017selective,feng2023towards}.  
    
    \item We propose two confidence scores based on the raw logits (v.s. the normalized logits, i.e., softmax responses), inspired by the notion of margins (\cref{subsec: DNN margins as confidence scores}). Through careful analysis (\cref{Sec: analysis of selective scores}) and extensive experiments (\cref{Sec: Experiments}), we show that our margin-based confidence scores are more reliable for generalized SC on various dataset-classifier combinations, even under moderate distribution shifts. 
\end{itemize}

\vspace{-0.5em}
\section{Technical background and related work}
\label{Sec: related works}
\subsection{Selective classification (SC)}
\label{subsec: selective classifiers definition}
Consider a multiclass classification problem with input space $\gX \in \R^n$,  label space $\gY = \set{1, \dots, K}$, and data distribution $\mc D_{\mc X, \mc Y}$ on $\mc X \times \mc Y$. A selective classifier $(f, g)$ consists of a predictor $f: \gX \rightarrow \R^K$ and a selector $g: \gX \rightarrow \{ 0, 1 \}$ and works as follows: 
\begin{align}
\label{eq: def of selective classifier}
    (f, g)(\mb x) \triangleq 
    \begin{cases}
        f(\mb x) &  \text{if}\; g(\mb  x) = 1  \\
        \text{abstain}  &  \text{if} \; g(\mb x) = 0
    \end{cases}, 
\end{align}
for any input $\mb x \in \mc X$. Typical selectors $g$ take the form:  
\begin{equation}
\label{eq: def selector}
    g_{s, \gamma}(\mb x) = \mathbbm{1}[s(\mb x) > \gamma],
\end{equation}
where $s(\mb x)$ is a \emph{confidence-score} function, and $\gamma$ is a tunable threshold for selection.

\subsection{Prior work in SC}
\label{subsec: DNN confidence scores}
For a given selective classifier $
(f, g_{s, \gamma})$, its SC performance is often characterized by two quantities: 
\begin{gather}
\begin{aligned}
\text{(\textbf{coverage})} \; \phi_{s, \gamma} & = \bb E_{\mc D_{\mc X, \mc Y}} \brac{ g_{s, \gamma}(\mb x)}, \quad & \text{(higher the better)}, \\
\text{
(\textbf{selection risk})} \; R_{s, \gamma} & = \bb E_{\mc D_{\mc X, \mc Y}} \brac{\ell(f(\mb x), y) g_{s, \gamma} (\mb x)} / \phi_{s, \gamma}, \quad & \text{(lower the better)},  
\label{eq:selection_risk_def}
\end{aligned}   
\end{gather}
Because a high coverage typically comes with a high selection risk, there is always a need for risk-coverage tradeoff in SC. Most of the existing work considers $\ell$ to be the standard $0/1$ classification loss~\citep{chow1970optimum, el2010foundations, geifman2018bias}, and we also follow this convention in this paper. A classical cost-based formulation is to optimize the risk-coverage (RC) tradeoff~\citep{chow1970optimum}

\begin{equation}
\label{eq: cost based formulation}
\min\nolimits_{f, g_{s, \gamma}} \; \bb E_{\mc D_{\mc X, \mc Y}} \brac{\ell(f(\mb x), y)
 g_{s, \gamma} (\mb x)} + \eps \bb E_{\mc D_{\mc X, \mc Y}} \brac{ 1 - g_{s, \gamma}(\mb x)}
\quad \equiv \quad \min\nolimits_{f, g_{s, \gamma}} \; R_{s,\gamma} \phi_{s, \gamma} - \eps \phi_{s, \gamma} 
,  
\end{equation}
where $\eps \in [0, 1] $ is the cost of making a rejection. The optimal selective classifier for this formulation is~\citep{chow1970optimum,franc2023optimal}: 
\begin{gather}
\begin{aligned}
\label{eq: optimal SC rule}
    f^*  = \argmin\nolimits_{\wh{y} \in \mc Y} \sum\nolimits_{y \in \mc Y} p(y | \mb x) \ell\paren{\wh{y}, y}, \quad \text{and} \quad 
    g^*  =  \mathbbm{1}[-\min\nolimits_{\wh{
y} \in \mc Y} \sum\nolimits_{y \in \mc Y} p(y | \mb x) \ell\paren{\wh{y}, y} > -
\eps],  
\end{aligned}
\end{gather}
where $f^*$ is the Bayes optimal classifier and depends on the posterior probabilities $p(y | \mb x)$ for all $y \in \mc Y$, which are hard to obtain in practice. Moreover, solutions to two constrained formulations for the RC tradeoff, 
\begin{equation}
\label{Eq. constrained specification of SC 1}
\min\nolimits_{f, g_{s,\gamma}} R_{s,\gamma},  \; \st \; \phi_{s,\gamma} \ge \omega \quad 
\text{and} \quad \max\nolimits_{f, g_{s,\gamma}} \phi_{s,\gamma},  \; \st \; R_{s,\gamma} \le \lambda,  
\end{equation}
also depend on the posterior probabilities~\citep{pietraszek2005optimizing,geifman2017selective,franc2023optimal,el2010foundations}.

\paragraph{Training-based scores}
Due to the intractability of true posterior probabilities in practice, many previous methods focus on learning effective confidence-score functions from training data. They require access to training data and learn parametric score functions, often under cost-based/constrained formulations and their variants for the RC tradeoff. This learning problem can be formulated together with~\citep{chow1970optimum,pietraszek2005optimizing,grandvalet2008support,el2010foundations,cortes2016boosting,geifman2019selectivenet,liu2019deep,huang2022self,gal2016dropout,lakshminarayanan2017simple, geifman2018bias,maddox2019simple,dusenberry2020efficient,lei2014classification,villmann2016self,corbiere2019addressing} or separately from training the classifier~\citep{jiang2018trust,fisch2022calibrated,franc2023optimal}. However, \citet{feng2023towards} has recently shown that these training-based scores do not outperform simple non-training-based scores described below. 

\paragraph{Manually designed (non-training-based) scores}
This family works with any given classifier and does not assume access to the training set. This is particularly attractive when it comes to modern pretrained large DNN models, e.g., CLIP~\citep{radford2021learning}, Florence~\citep{yuan2021florence}, and GPTs~\citep{brown2020language}, for which obtaining the original training data and performing retraining are prohibitively expensive, if not impossible, to typical users. \cref{alg: confidence-based sc} shows a typical use case of SC with non-training-based scores. Different confidence scores have been proposed in the literature. For example, for support vector machines (SVMs), confidence margin (the difference of the top two raw logits) has been used as a confidence score~\citep{fumera2002support,franc2023optimal}; see also \cref{subsec: DNN margins as confidence scores}. For DNN models, \emph{which is our focus}, confidence scores are popularly defined over the \emph{softmax responses} (SRs). Assume that $\mb z \in \R^K$ contains the raw logits (RLs) and $\sigma$ is the softmax activation. The following three confidence-score functions 
\begin{gather}
\label{eq: SR based SC scores}
   \begin{aligned}
    SR_{\text{max}} (\mb z) \triangleq \;  \max_i \sigma(z_i), \quad 
    SR_{\text{doctor}} (\mb z) \triangleq \;   1 - 1/ \norm{\sigma(
\mb z)}^2_2, \quad 
    SR_{\text{ent}} (\mb z) \triangleq \;  \sum\nolimits_i \sigma(z_i)\log \sigma(z_i), 
    \end{aligned} 
\end{gather}
are popularly used in recent work, e.g., \cite{feng2023towards,granese2021doctor,xia2022augmenting}. Although simple, $SR_{\text{max}}$ can easily beat existing training-based methods~\citep{feng2023towards}. On the other hand, these SR-based score functions generally follow the plug-in principle by assuming that SRs approximate posterior probabilities well~\citep{franc2023optimal}. Unfortunately, this assumption often does not hold in practice, and bridging this approximation gap is a major challenge for confidence calibration~\citep{guo2017calibration,nixon2019measuring}. However, \citet{zhu2022rethinking} reveals that recent calibration methods may even degrade SC performance.  

\begin{algorithm}[tb]
\caption{Non-training-based selective classification}
\label{alg: confidence-based sc}
\begin{algorithmic}[1]
\Require A pretrained classifier $f$; a score function $s$; a small calibration dataset $ Z^{cali} \sim_{iid} \gD^{cali}_{\mc X, \mc Y}$
\State $\forall (\mb x_i, y_i) \in Z^{cali}$, compute $s(\mb x_i)$ and $\ell(f(\mb x_i), y_i)$
\State Determine a threshold $\gamma$ according to the coverage or selection-risk target 
\State Deploy the selector $g_{s, \gamma}$ based on \cref{eq: def selector}.
\end{algorithmic}%
\end{algorithm}

\subsection{SC under distribution shifts: generalized SC}
\label{sec:sc_dist_shift}

In this paper, we consider SC under distribution shifts, or \emph{generalized selective classification}. Shifts between training and deployment distributions are common in practice and can often cause performance drops in deployment~\citep{quinonero2008dataset,rabanser2019failing,koh2021wilds}, raising reliability concerns for high-stakes applications in the real world. In this paper, we use the term \emph{distribution shifts} to cover both covariate and label shifts---perhaps the most prevalent forms of distribution shifts (see the beginning of \cref{Sec: introduction} for their definitions)---jointly. Although the basic set-up for our generalized SC framework remains the same as that of \cref{eq: def of selective classifier,eq: def selector}, we need to modify the definitions for selection risk and coverage in \cref{eq:selection_risk_def} to take into account potential distribution shifts: 
\begin{gather} 
\begin{aligned} 
\text{(\textbf{coverage})} \; \phi_{s, \gamma}  = \bb E_{\mc D'_{\mc X, \mc Y'}} \brac{ g_{s, \gamma}(\mb x)}, \quad \text{and} \quad 
\text{(\textbf{selection risk})} \; R_{s, \gamma}  = \bb E_{\mc D'_{\mc X, \mc Y'}} \brac{\ell(f(\mb x), y) g_{s, \gamma} (\mb x)} / \phi_{s, \gamma}
\label{eq: generalized SC RC}, 
\end{aligned}   
\end{gather}
where $\mc D_{\mc X, \mc Y}$ is the original data distribution, $\mc D'_{\mc X, \mc Y'}$ is the shifted distribution---$\mc Y'$ may not be the same as $\mc Y$ due to potential label shifts.\footnote{We assume no \emph{outliers} in generalized SC---samples that do not follow any specific statistical patterns---during deployment, i.e., they are already detected and removed after separate data preprocessing steps. This allows us to properly define the coverage and selection risk.} 

\begin{algorithm}[tb]
\caption{Typical OOD detection pipeline (e.g., \citet{sun2021react})}
\label{alg: OOD detector}
\begin{algorithmic}[1]
\Require An OOD score function $s_{OOD}$; an In-D calibration dataset $X^{in} \sim_{iid} \gD_{\mc X}$ and an OOD calibration dataset $X^{ood} \sim_{iid} \gD^{OOD}_{\mc X}$

\State $\forall \mb x_i \in X^{in}$ and $\forall \mb x_j \in X^{ood}$, compute $s_{OOD}(\mb x_i)$ and $s_{OOD}(\mb x_j)$.
\State Compute a threshold $\gamma_{\text{OOD}}$ using \cref{eq: def of ood detection} by problem-specific target requirements, e.g., a target TPR (true positive rate) value.
\State Deploy the OOD detector according to \cref{eq: def of ood detection}.
\end{algorithmic}%
\end{algorithm}
\paragraph{Out-of-distribution (OOD) detection as a weak form of generalized SC} 
The goal of OOD detection is to detect and exclude OOD samples~\citep{yang2021generalized}. An ideal OOD detector $G(\mb x)$ should perfectly separate In-D and OOD samples: 
\begin{align} 
    \label{eq: def of ood detection}
    G(\mb x) =
    \begin{cases}
        0 \; (\text{i.e., excluded}) &  \text{if}\; \mb  x  \sim \gD^{\text{OOD}}_{\mc X}  \\
        1 \; (\text{i.e., kept}) &  \text{if} \; \mb x \sim \gD_{\mc X} 
    \end{cases}, 
    \quad \text{which is often realized as} \quad  G(\mb x) = \mathbbm{1}[s_{\text{OOD}}(\mb x) > \gamma]. 
\end{align}
Here, $s_{\text{OOD}}(\cdot)$ is a confidence-score function indicating the likelihood that the input is an In-D sample, and $\gamma$ is again a tunable cutoff threshold. Although by the literal meaning of OOD both covariate and label shifts are covered by $\gD^{\text{OOD}}_{\mc X}$, the literature on OOD detection focuses mainly on detecting \emph{label-shifted} samples, i.e., covariate-shifted $\gD^{\text{OOD}}_{\mc X}$ induced by label shifts~\citep{liu2020energy, sun2021react, wang2022vim, sun2022out}. OOD detection is commonly motivated as an approach to achieving reliable predictions: under the assumption that $\gD^{\text{OOD}}_{\mc X}$ is induced by label shifts only, any OOD samples will cause misclassification and hence should be excluded---clearly aligned with the goal of SC. \cref{alg: OOD detector} shows the typical use case of OOD (label-shift) detectors, and its similarity to SC shown in \cref{alg: confidence-based sc} is self-evident. However, OOD detection clearly aims for less than generalized SC in that: (1) even if the OOD detection is perfect, misclassified samples---either as In-D or due to distribution shifts---by imperfect classifiers are not rejected, and (2) practical OOD detectors may fail to perfectly separate In-D and OOD samples, OOD detected but correctly classified In-D samples are still rejected, hurting the classification performance on the selected samples; see \cref{App: eva OOD} for an illustrative example. Therefore, if we are to achieve reliable predictions by excluding samples that are likely to cause errors, we should directly follow the generalized SC instead of the OOD detection formulation. 

\begin{wrapfigure}{r}{0.3\textwidth}
    \vspace{-1em}
    \includegraphics[width=0.3\textwidth]{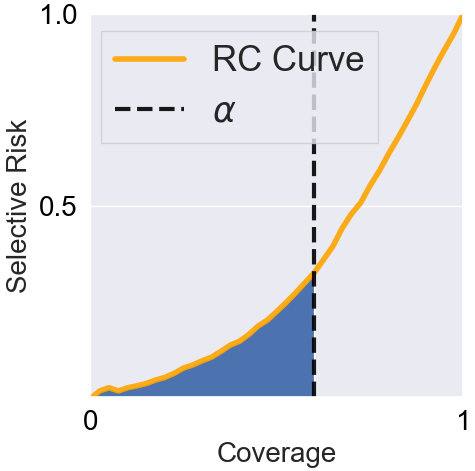}
    \vspace{-2em}
    \caption{Visualization of the \emph{normalized AURC-$\alpha$}---the area in blue divided by the coverage value $\alpha$.}
    \label{Fig: Demo AURC}
    \vspace{-1em}
\end{wrapfigure}

\paragraph{Other related concepts}
\hy{Besides OOD detection, OOD generalization focuses on correctly classifying In-D and covariate-shifted samples, without considering prediction confidence and selection to improve prediction reliability; open-set recognition (OSR) focuses on correctly classifying In-D samples, as well as flagging label-shifted samples; see \citet{geng2020recent} for a comprehensive review. In contrast, generalized SC covers all In-D, label-shifted, and covariate-shifted samples, the widest coverage compared to these related concepts, and targets the most practical and pragmatic metric---classification performance on the selected samples.} 

\paragraph{Prior work on SC with distribution shifts}
Although the existing literature on SC is rich~\citep{zhang2023survey}, research work that considers SC with potential distribution shifts is very recent and focuses only on \emph{label shifts}: \citet{xia2022augmenting,kim2023unified} perform In-D SC and OOD (label shift) detection together with a confidence score that combines an SC score and an OOD score, but they still evaluate the performance of In-D SC and OOD detection \emph{separately}. \hy{\cite{muller2023finding,cattelan2023selective} empirically show that existing OOD scores are not good enough for SC tasks with covariate/label-shifted samples; \cite{cattelan2023selective} proposes ways to refine these scores with the help of additional datasets to optimize performance. \cite{franc2024scod} provides theoretical insights on SC with In-D and label-shifted samples. In contrast, we focus on identifying better confidence scores for generalized SC---that covers both In-D and covariate/label-shifted samples and maximizes the utility of the classifier, and unify the evaluation protocol (see \cref{Subsec: eva SC OOD}). } 

\subsection{Evaluation of generalized SC}  
\label{Subsec: eva SC OOD} 
Since the goal of generalized SC is to identify and exclude misclassified samples,  for performance evaluation at a fixed cutoff threshold $\gamma$, it is natural to report the coverage---the portion of samples accepted, and the corresponding selection risk---``accuracy'' (taken broadly) on accepted samples. It is clear from \cref{eq: def of selective classifier,eq: def selector} that for a given pair of classifier $f$ and confidence-score function $s$, the threshold $\gamma$ can be adjusted to achieve different risk-coverage (RC) tradeoffs. By continuously varying $\gamma$, we can plot a \emph{risk-coverage (RC) curve}~\cite{el2010foundations,franc2023optimal} to profile the SC performance of $(f, s)$ throughout the entire coverage range $\phi_{\gamma} \in [0, 1]$; see \cref{Fig: Demo AURC} for an example. Generally, the lower the RC curve, the better the SC performance. To obtain a summarizing metric, it is natural to use the \emph{area under the RC curve} (AURC)~\citep{el2010foundations,franc2023optimal}. We note that the RC curve and the AURC are also the most widely used evaluation metrics for classical SC---which is not surprising, as the goal of classical SC aligns with that of generalized SC, although generalized SC also allows distribution shifts. 

For typical high-stakes applications, such as medical diagnosis, low selection risks are often prioritized over high coverage levels. So, in addition to RC curves and AURC, we also report several partial AURCs to account for potential different needs---\emph{normalized AURC-$\alpha$}, where $\alpha$ specifies the coverage level, and we normalize the partial area-under-the-curve by the corresponding $\alpha$ so that different partial levels can be cross-compared; see \cref{Fig: Demo AURC} for illustration. 

Note that RC curves, and hence the associated AURCs and normalized AURC-$\alpha$ also, depend on the $(f, s)$ pair. So, if the purpose is to \emph{compare different confidence-score functions}, $f$ should be fixed. \citet{feng2023towards} has recently pointed out the abuse of this crucial point in recent training-based SC methods. Thus, it is worth stressing that we \emph{always take and fix pretrained $f$'s} when making the comparison between different score functions.

\subsection{Few words on implementing \cref{alg: confidence-based sc} in practice}
In the practical implementation of generalized SC for high-stakes applications after \cref{alg: confidence-based sc}, it is necessary to select a cutoff threshold $\gamma$ based on a calibration set to meet the target coverage, or more likely the target risk level. However, in this paper, we follow most existing work on SC and do not touch on issues such as how the calibration set should be constructed and how the threshold should be selected---we leave these for future work. Our evaluation here, again, as most existing SC work, is only about the \emph{potential} of specific confidence-score functions for generalized SC, measured by the RC curve, AUPC, and normalized AURC-$\alpha$'s, directly on test sets that consist of In-D, OOD, and covariate-shifted samples.

\section{Our method---margins as confidence scores for generalized SC}
\label{Sec: geo margin definition}

Our goal is to design effective confidence-score functions for generalized SC. Again, our focus is on non-training-based scores that can work on any pretrained classifier $f$ without access to the training data. 

\subsection{Scale sensitivity of SR-based scores}
\label{subsec: SR-based score is problematic}
As discussed in \cref{subsec: DNN confidence scores}, most manually designed confidence scores focus on DNN models and are based on softmax responses (SRs), assuming that SRs closely approximate true posterior probabilities---closing such approximation gaps is the goal of confidence calibration. However, effective confidence calibration remains elusive~\citep{guo2017calibration,nixon2019measuring}, and the performance of SR-based score functions is sensitive to the scale of raw logits and hence that of SRs, as explained below. 

\begin{figure*}[!htbp]
\centering
\resizebox{1\linewidth}{!}{%
\begin{tabular}{c c c c}

\centering
\includegraphics[width=0.24\textwidth
]{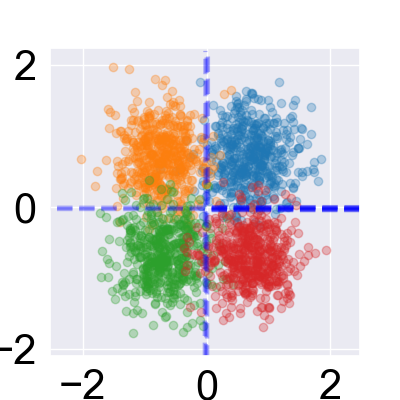}%
&\includegraphics[width=0.24\textwidth]{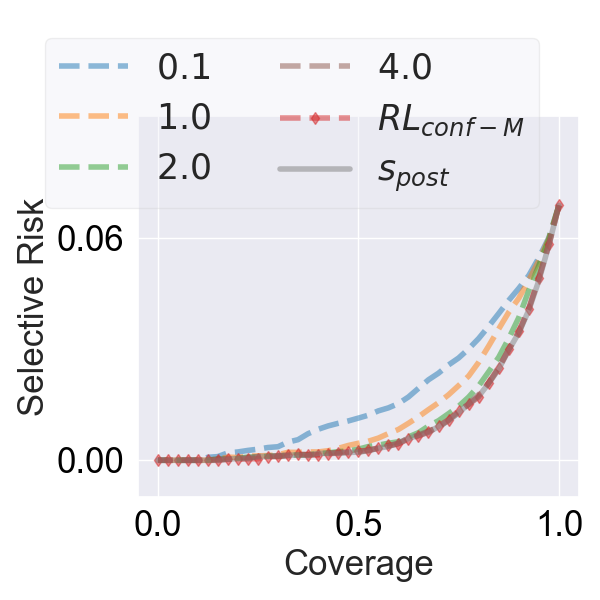}%
&\includegraphics[width=0.24\textwidth]{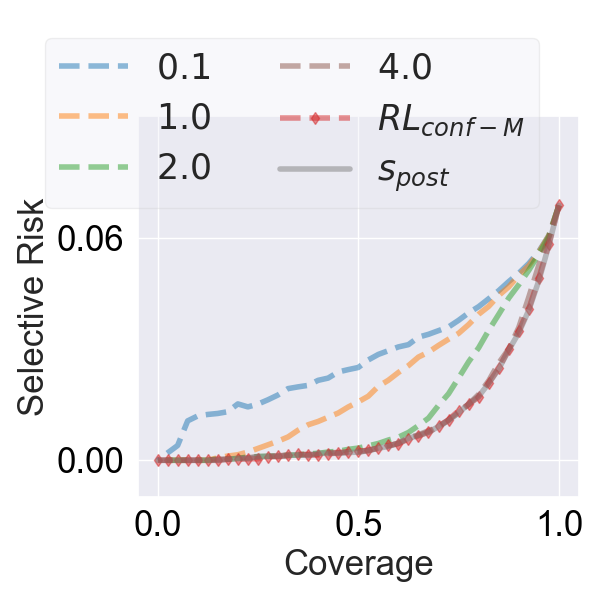}%
&\includegraphics[width=0.24\textwidth]{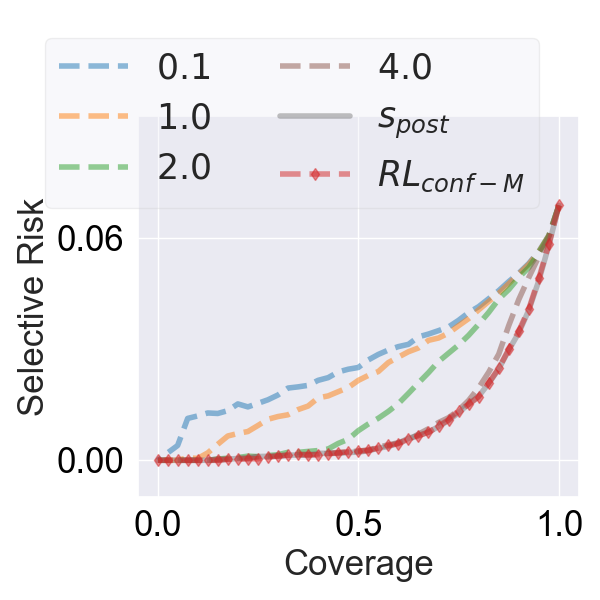}%
\\
\small{\textbf{(a)} Data and classifier visualization}%
&\small{\textbf{(b)} $SR_{\text{max}}$}%
&\small{\textbf{(c)} $SR_{\text{doctor}}$}%
&\small{\textbf{(d)} $SR_{\text{ent}}$}%
\end{tabular}}%
\caption{RC curves for \textbf{(b)} $SR_{\text{max}}$, \textbf{(c)} $SR_{\text{doctor}}$, and \textbf{(d)} $SR_{\text{ent}}$, calculated based on scaled (by factor $0.1$, $1.0$, $2.0$, and $4.0$, respectively) raw logits from the optimal $4$-class linear classifier using data shown in \textbf{(a)}. The RC curves for $RL_{\text{conf-M}}$ and $s_{\text{post}}$ are also plotted for reference, where $RL_{\text{conf-M}}$ is one of our proposed confidence-score functions.}%
\vspace{-0.5em}
\label{Fig: SVM scaled RC}
\end{figure*}

\paragraph{A quick numerical experiment} 

Consider a $4$-component mixture-of-Gaussian distribution with means $\mb w_1 = [\sqrt{2}/2, \sqrt{2}/2]^\T$, $\mb w_2 = [- \sqrt{2}/2, \sqrt{2}/2]^\T$, $\mb w_3 = [- \sqrt{2}/2, - \sqrt{2}/2]^\T$, $\mb w_4 = [\sqrt{2}/2, - \sqrt{2}/2]^\T$, equal variance $0.15\times \mb I$, and equal weight $1/4$. If we treat each component of the mixture as a class and consider the resulting $4$-class classification problem, it is easy to see that the optimal $4$-class linear classifier is $
f(\mb x) = [\mb w_1, \mb w_2, \mb w_3, \mb w_4]^\T \mb x$, with the decision rule $\argmax_{j \in \{1, 2, 3, 4\}} \mb w_j^\T \mb x$; see \cref{Fig: SVM scaled RC} \textbf{(a)} for visualization of the data distribution and decision boundaries (i.e., the lines $\mb x_1 = 0$ and $\mb x_2 = 0$). Moreover, this $f(\mb x)$ is also a Bayes optimal classifier as well as the maximum a posterior (MAP) classifier, for our particular problem here. Now, given any input $\mb x$, we consider scaled raw logits $\lambda f(\mb x)$ for different scale factors $\lambda = 0.1, 1, 2, 4$ and plot the resulting RC curves for $SR_{\text{max}}$, $SR_{\text{doctor}}$, and $SR_{\text{ent}}$, respectively; see \cref{Fig: SVM scaled RC}\textbf{(b)-(d)}. For reference, we also include the RC curves based on the true posterior probabilities (denoted as $s_{\text{post}}$), which are available for our simple data model here. We can observe that for SR-based functions ($SR_{\text{max}}$, $SR_{\text{doctor}}$, and $SR_{
\text{ent}}$), their RC curves and hence the associated AURC's vary as $\lambda$ changes, and these curves approach a common curve ($RL_{\text{conf-M}}$, which we will explain below) as $\lambda$ becomes large. 

\paragraph{Why it happens?} 
The above observations are not incidental. To see why the curves change with respect to $\lambda$, note that for a given test set $\set{\mb x_i}$ and a fixed classifier $f$, the RC curve for any score function $s$ is fully determined by the ordering of $s(\mb x_i)$'s~\citep{franc2023optimal}. But this ordering is sensitive to the scale of the raw logits for all three SR-based score functions: $SR_{\text{max}}$, $SR_{\text{doctor}}$, and $SR_{\text{ent}}$. Take $s = SR_{\text{max}}$ as an example and consider any sample $\mb x$ with its corresponding raw logits $\mb z$ sorted in descending order (i.e., $z^{(1)} \ge z^{(2)} \ge \cdots$) without loss of generality. Then for any scale factor $\lambda > 0$ applied to $\mb z$, we have the score  
\begin{equation}
    SR_{\text{max}} \triangleq e^{\lambda z^{(1)}}/\sum\nolimits_{j} e^{\lambda z^{(j)}} 
    =  1/\sum\nolimits_{j} e^{\lambda (z^{(j)} - z^{(1)})}
    = \exp\paren{-\log \sum\nolimits_{j} e^{\lambda (z^{(j)} - z^{(1)})}}. 
\end{equation}
This means that the score is determined by all the scaled \emph{logit gaps} $\lambda (z^{(j)} - z^{(1)})$'s. Moreover, due to the inner exponential function, small gaps gain more emphasis as $\lambda$ increases, and all gaps receive increasingly more emphasis as $\lambda$ decreases. Such a shifted emphasis can easily change the order of scores for two data samples, depending on how different their raw logits are distributed. Clearly, $e^{\lambda z^{(1)}}/\sum\nolimits_{j} e^{\lambda z^{(j)}}  = 1/K$ as $\lambda = 0$. We can also make similar arguments for $SR_{\text{doctor}}$ and $SR_{\text{ent}}$. Next, for the common asymptotic curve as $\lambda \to \infty$, we can show the following (proof is deferred to \cref{App: sr-based scores limits}): 
\begin{lemma} \label{lem:sr_asymp}
    Consider the raw logits $\mb z$, and without loss of generality assume that they are ordered in descending order without any ties, i.e., $z^{(1)} > z^{(2)} > \cdots $. We have that as $\lambda \to \infty$, 
    \begin{align*}
        SR_{\mathrm{max}} (\lambda\mb z) \sim \exp(- e^{\lambda (z^{(2)}- z^{(1)})}), \quad  SR_{\mathrm{doctor}} (\lambda\mb z) \sim 1- \exp\paren{2 e^{\lambda (z^{(2)} - z^{(1)} )}},  \quad SR_{\mathrm{ent}}(\lambda\mb z) \sim - e^{\lambda(z^{(2)} - z^{(1)})},  
    \end{align*}
    where $\sim$ means asymptotic equivalence. In particular, all the asymptotic functions increase monotonically with respect to $z^{(1)} - z^{(2)}$. 
\end{lemma}
This implies that the asymptotic RC curve as $\lambda \to \infty$ for all three score functions is fully determined by the score function $z^{(1)} - z^{(2)}$! 

\paragraph{Implications} 
The sensitivity of the RC curves, and hence of the performance, of these SR-based scores to the scale of raw logits is disturbing. \emph{It implies that one can simply change the overall scale of the raw logits---which does not alter the classification accuracy itself---to claim better or worse performance of an SR-based confidence-score function for selective classification, making the comparison of different SR-based scores shaky}. Unfortunately, between the limiting cases $\lambda \to 0$ and $\lambda \to \infty$, there is no canonical scaling. 

\subsection{Our method: margin-based confidence scores}
\label{subsec: DNN margins as confidence scores}
To avoid the scale sensitivity caused by the softmax nonlinearity, it is natural to consider designing score functions directly over the raw logits. To this end, we revisit ideas in support vector machines (SVMs). 

\paragraph{Margins in SVMs}
In linear SVMs for binary classification, the classifier takes the form $f(\mb x) = \sign(\mb w^\T \mb x + b)$ and the confidence in classifying a sample $\mb x$ can be assessed by its distance from the supporting hyperplane~\citep{fumera2002support,franc2023optimal}: $\abs{\mb{w}^\T \mb x + b} / \norm{\mb{w}}_2$, which is called the \emph{geometric margin}; see \cref{App: Multiclass SVM} for a detailed review. We can extend the idea to $K$-class linear SVMs. Following the popular joint multiclass SVM formulation~\citep{crammer2001algorithmic}, we consider a linear classifier $f(\mb x) = \mb W^\T \mb x + \mb b$. Here, $\mb W$ and $\mb b$ induce $K$ hyperplanes, and we can define the signed distance of any sample $\mb x$ to the $i$-th hyperplane as: $\paren{\mb w_i^\T \mb x + b_i}/\norm{\mb w_i}$ ($\mb w_i$ denotes the $i$-th column of $\mb W$ and $b_i$ the $i$-th element of $\mb b$), generalizing the definition for the binary case. However, a single signed distance makes little sense for assessing the classification confidence in multiclass cases, given the typical \texttt{argmax} decision rule---e.g., the largest signed distance can be negative. Instead, comparing the distances to all decision hyperplanes seems more reasonable. Thus, we can consider the following \emph{geometric margin} as a confidence-score function: 
\begin{align}
\label{eq: svm relative geo margin}
    (\mb w_{y'}^\T \mb x + b_{y'})/\norm{\mb w_{y'}}_2 - \max\nolimits_{j \in \set{1, \dots, K}\setminus {y'}} \; (\mb w_j^\T \mb x + b_j)/\norm{\mb w_j}_2 \; \text{,}
\end{align}
where $y' \in \argmax_{j \in \set{1, \dots, K}} (\mb w_j^\T \mb x + b_j)/\norm{\mb w_j}_2$.
In other words, it is the difference between the top two signed distances of $\mb x$ to all $K$ hyperplanes. Intuitively, the larger the geometric margin, the more confident the classifier is in classifying the sample following the largest signed distance---\emph{a clearer winner earns more trust}. Although the interpretation is intuitive, the geometric margin is not popularly used in multiclass SVM formulations, likely due to its non-convexity. Instead, a popular proxy for the geometric margin is the convex \emph{confidence margin}: 
\begin{align}
\label{eq: svm confidence margin}
(\mb w_{y'}^\T \mb x + b_{y'}) - \max\nolimits_{i \in \set{1, \dots, K}\setminus {y'}} \; (\mb w_i^\T \mb x + b_i), 
\end{align}
with the decision rule $y' \in \argmax_{j \in \{1, \cdots, K\}} \mb w_j^\T \mb x + b_j$; see \cref{App: Multiclass SVM}. Despite its numerical convenience, the confidence margin loses geometric interpretability compared to the geometric margin, and it can be sensitive to the scaling of $\mb w_j$. We study both margins in this paper. 

\paragraph{Margins in DNNs}
To extend the idea of margins to a DNN classifier $f_{\mb \theta}(\mb x)$ parameterized by $\mb \theta$, we view all but the final linear layer as a feature extractor, denoted as $f_{\mb \theta}^e$. So, for each sample $\mb x$, the logit output takes the form $\mb z = \mb W^\T f_{\mb \theta}^e(\mb x)+ \mb b$, and thus the signed distance of the representation $f_{\mb \theta}^e(\mb x)$ to each decision hyperplane in the representation space is: $d_j =  (\mb{w}_j^\T f_{\mb \theta}^e(\mb x) + b_j) / \norm{\mb{w}_j}_2\quad  \forall j \in \{1, \dots, K\}$. 
Assume sorted signed distances and logits, i.e., $d^{(1)} \ge d^{(2)} \ge \ldots \ge d^{(K)}$ and $z^{(1)} \ge z^{(2)} \ge \ldots \ge z^{(K)}$. The \emph{geometric margin} and the \emph{confidence margin} are defined as
\begin{align} 
\label{eq: geo_raw margin}
   RL_{\text{geo-M}} \triangleq d^{(1)} - d^{(2)} \quad \text{and} \quad RL_{\text{conf-M}} \triangleq z^{(1)} - z^{(2)}, \text{respectively}. 
\end{align}
Note that both $RL_{\text{geo-M}}$ and $RL_{\text{conf-M}}$ are computed using the \emph{raw logits without softmax normalization}; $z$'s and $d$'s may not have the same ordering due to the scale of $\norm{\mb{w}_j}_2$.
In fact, $RL_{\text{conf-M}}$ is applied in \citet{lecun1989handwritten} to formulate an empirical rejection rule for a handwritten recognition system, although no detailed analysis or discussion is given on why it is effective. Despite the simplicity of these two notions of margins, we have not found prior work that considers them for SC except for \citet{lecun1989handwritten}.  

\paragraph{Scale-invariance property}
An attractive property of margin-based score functions is that their SC performance is \emph{invariant} w.r.t. the scale of raw logits. This is because changing the overall scale of the raw logits does not change the order of scores assigned by either the geometric or the confidence margin. In this regard, margin-based score functions are much more preferred and reliable than SR-based scores for SC. Another interesting point is that the limiting curve depicted in \cref{Fig: SVM scaled RC}\textbf{(b)-(d)} is induced by the confidence margin, as is clear from \cref{lem:sr_asymp} and the discussion following it.

\begin{figure*}[!tb]
\centering
\resizebox{1\linewidth}{!}{%
\begin{tabular}{c c c c}
\cdashline{1-4}
\vspace{-0.9em}
\\
\multicolumn{4}{c}{Case 1}\\
\includegraphics[width=0.22\textwidth]{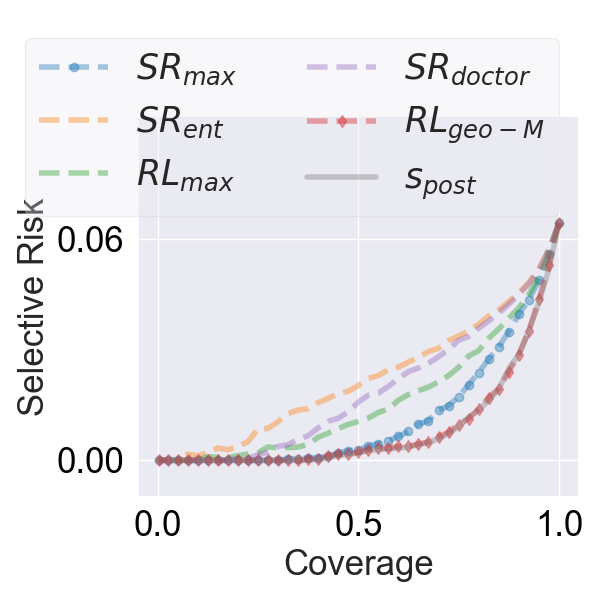}%
&\includegraphics[width=0.22\textwidth]{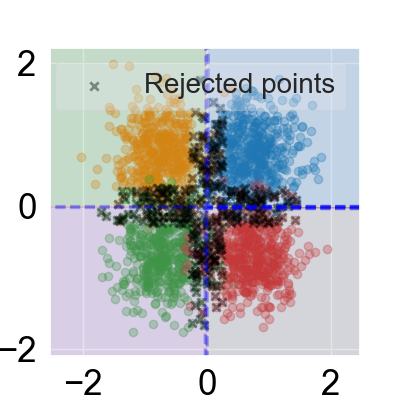}%
&\includegraphics[width=0.22\textwidth]{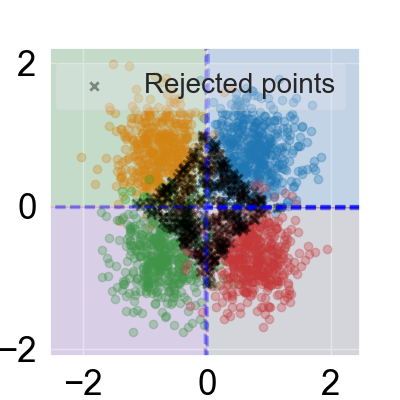}%
&\includegraphics[width=0.22\textwidth]{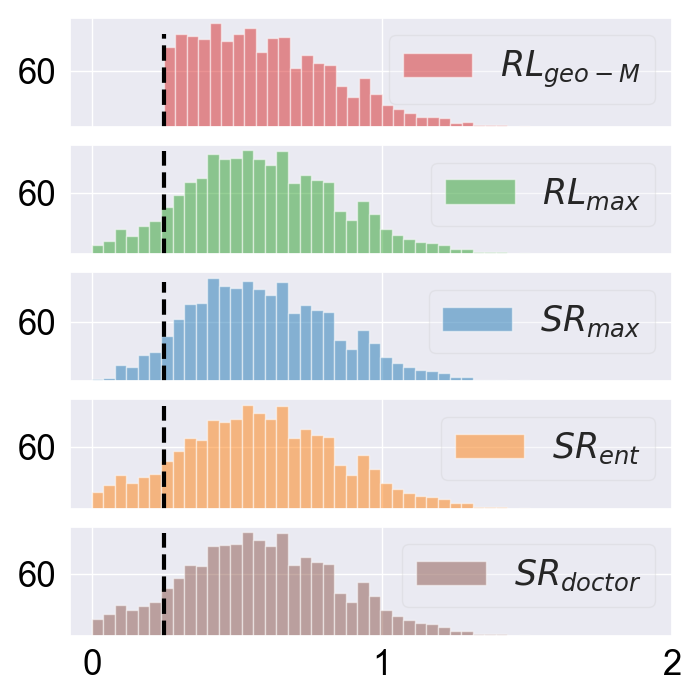}%
\\
\small{{\textbf{(a-1)}}}%
&\small{{\textbf{(b-1)} }}%
&\small{{\textbf{(c-1)} }}%
&\small{{\textbf{(d-1)}}}%
\vspace{0.3em}
\\
\cdashline{1-4}
\vspace{-0.9em}
\\
\multicolumn{4}{c}{\textcolor{orange}{Case 2}}
\vspace{-0.2em}
\\
\includegraphics[width=0.22\textwidth]{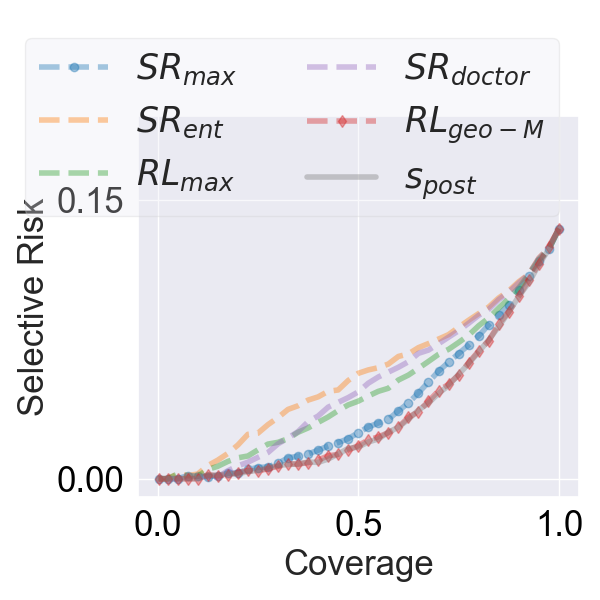}%
&\includegraphics[width=0.22\textwidth]{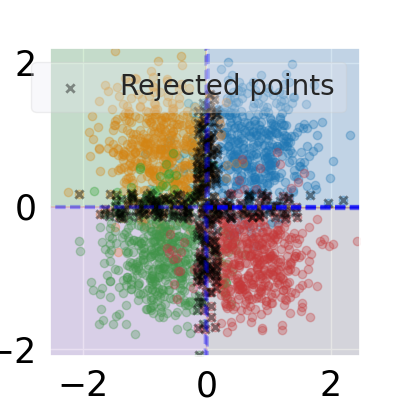}%
&\includegraphics[width=0.22\textwidth]{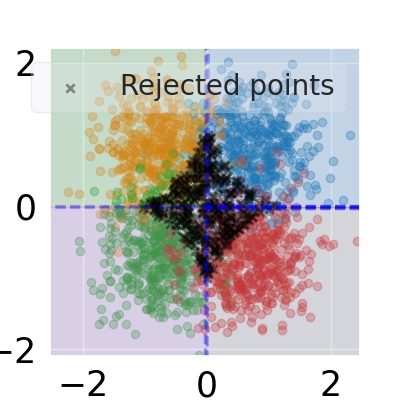}%
&\includegraphics[width=0.22\textwidth]{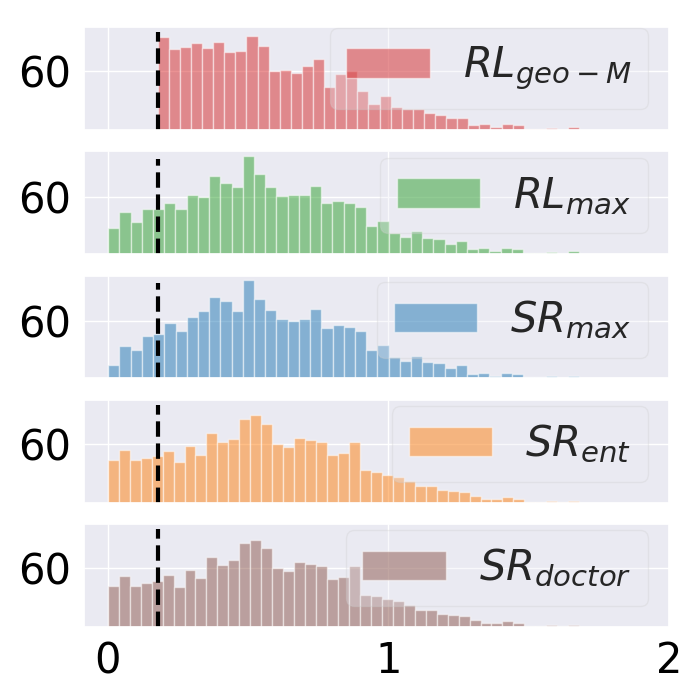}%
\\
\small{\textcolor{orange}{\textbf{(a-2)}}}%
&\small{\textcolor{orange}{\textbf{(b-2)}}}%
&\small{\textcolor{orange}{\textbf{(c-2)}}}%
&\small{\textcolor{orange}{\textbf{(d-2)}}}%
\vspace{0.3em}
\\
\cdashline{1-4}
\vspace{-0.9em}
\\
\multicolumn{4}{c}{\textcolor{blue}{Case 3}}
\vspace{-0.2em}
\\
\includegraphics[width=0.22\textwidth]{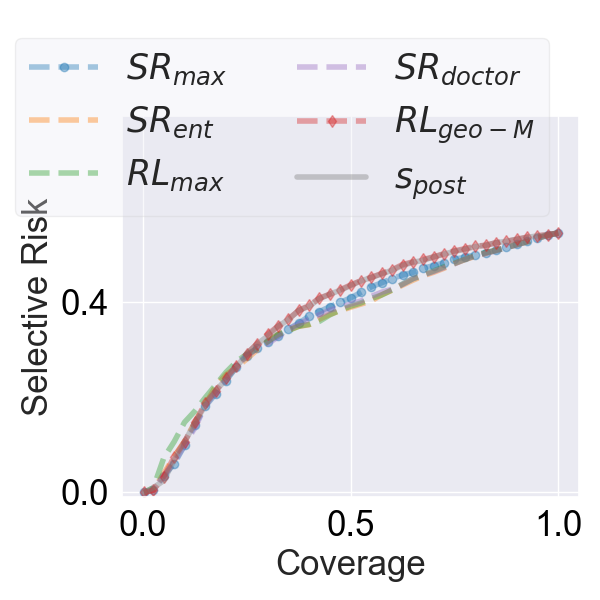}%
&\includegraphics[width=0.22\textwidth]{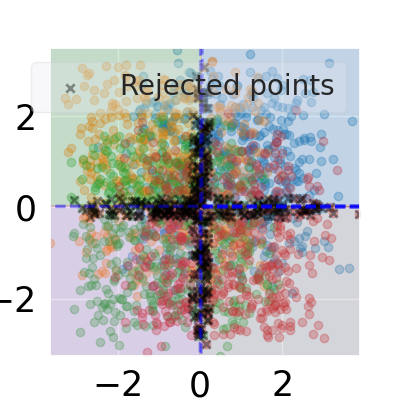}%
&\includegraphics[width=0.22\textwidth]{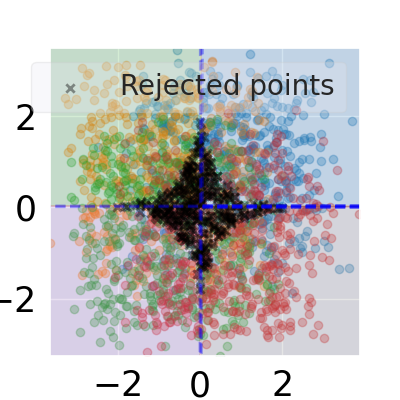}%
&\includegraphics[width=0.22\textwidth]{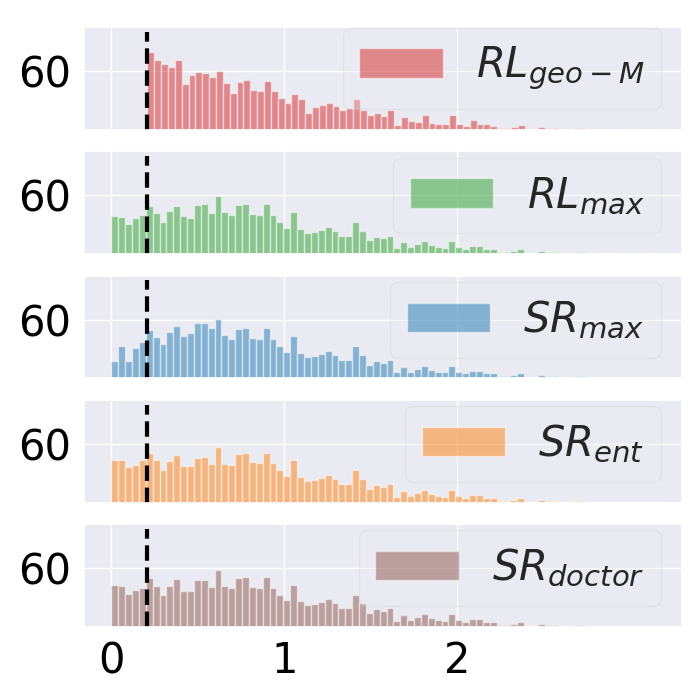}%
\\
\small{\textcolor{blue}{\textbf{(a-3)}}}%
&\small{\textcolor{blue}{\textbf{(b-3)}}}%
&\small{\textcolor{blue}{\textbf{(c-3)}}}%
&\small{\textcolor{blue}{\textbf{(d-3)}}}%
\vspace{0.3em}
\\
\cdashline{1-4}
\vspace{-0.9em}
\\
\end{tabular}}%
\caption{Further analysis of the numerical example in \cref{subsec: SR-based score is problematic}. Case 1, \textcolor{orange}{Case 2}, and \textcolor{blue}{Case 3} correspond to the original dataset in \cref{subsec: SR-based score is problematic}, the dataset after small perturbations, and the dataset after substantial perturbations, respectively. Here, \textbf{(a-)}'s are the RC curves achieved by different selection scores; \textbf{(b-)}'s are visualizations of the samples (one color per class), decision boundaries (dashed blue line) and the rejected samples (black crosses) at coverage $0.8$ by $RL_{\text{geo-M}}$; \textbf{(c-)}'s visualize the rejected samples (black crosses) at coverage $0.8$ by $SR_{\text{max}}$; and \textbf{(d-)}'s present the histogram of the robustness radius of the selected samples in by all score functions.}%
\vspace{-1em}
\label{Fig: SVM Toy Example Fig}
\end{figure*}
\subsection{Analysis of rejection patterns}
\label{Sec: analysis of selective scores}
We continue with the toy example in \cref{subsec: SR-based score is problematic} to show another major difference between the SR-based and the margin-based score functions---they have \emph{different rejection patterns for given coverage levels}. We will see that margin-based score functions induce favorable rejection patterns and can hence be used for reliable rejection even under moderate covariate shifts. For comparison, we also consider the maximum raw logit (denoted as $RL_{\text{max}}$) to show that a single logit in multiclass classification is not a sensible confidence score.  
\begin{description}[leftmargin=1em,nosep] 
    \item[Case 1:]   We use the same setup as in the numerical experiment in \cref{subsec: SR-based score is problematic} (see also \cref{Fig: SVM scaled RC}), and plot in \cref{Fig: SVM Toy Example Fig} \textbf{(a-1)} the RC curves induced by the various confidence-score functions\footnote{For the classifier consideblue, $RL_{\text{geo-M}}$ and $RL_{\text{conf-M}}$ have the same SC performance as $\|w\|_2 = 1$.}. It is clear that $RL_{\text{geo-M}}$ performs the best. To better understand the difference between $RL_{\text{geo-M}}$ and other score functions, we study their rejection patterns: we visualize in \cref{Fig: SVM Toy Example Fig} \textbf{(b-1)\&(c-1)} the samples rejected at $0.8$ coverage for $RL_{\text{geo-M}}$ and $SR_{\text{max}}$, respectively; see visualization of other score functions in \cref{App: SVM example}, whose rejection patterns are similar to that of $SR_{\text
    {max}}$. \emph{An iconic feature of $RL_{\text{geo-M}}$ is that it prioritizes rejecting samples closer to decision boundaries, whereas SR-based scores prioritize rejecting samples close to the origin}. Conceptually, the former rejection pattern is favorable, as the goal of SC is exactly to reject uncertain samples on which classifier's decisions can be shaky. More precisely, the difference in rejection patterns implies at least two things: (1) $RL_{\text{geo-M}}$ could be advantageous when most classification errors occur near the decision boundaries; (2) $RL_{\text{geo-M}}$ may be superior even when test samples have a moderate level of distribution shifts with respect to training. For example, when the test set has a slightly different $\mc D_{\mc X | \mc Y}$ than the training set (see \textbf{\textcolor{orange}{Cases 2} \& \textcolor{blue}{3}} below), mistaken samples due to the shift tend to be close to the decision boundaries and thus can be successfully rejected. \cref{Fig: SVM Toy Example Fig} \textbf{(d-1)} plots the histograms of the \emph{robustness radii} (i.e., the $\ell_2$ distance of a sample to the closest decision boundary) of selected samples at $0.8$ coverage, where the robustness radius quantitatively captures the extent of $\mc D_{\mc X | \mc Y}$ shift SC can tolerate: while the selected samples using $RL_{\text{geo-M}}$ uniformly have nonzero robust radii, all other score functions lead to zero robustness radii for the worst samples, implying sensitivity to $\mc D_{\mc X | \mc Y}$ shifts.\footnote{The intuition on why our notions of margins work for Type B errors is different: there since $\mb x$ assumes a label outside the known set, we expect no clear winner in the raw logits. }

    \item[\textcolor{orange}{Case 2:}] We keep the same setup as Case 1, except that small perturbations are added on all samples. The perturbations are drawn from a uniform distribution within the interval $[-0.5, 0.5]$ on each dimension of $\R^2$; see \cref{Fig: SVM Toy Example Fig} \textbf{\textcolor{orange}{(b-2)}}, where more samples of different classes are intermingled than before the perturbations are added. Although some misclassified samples have moved far into the bulks of other classes, most of them are still close to the decision boundaries. Therefore, $RL_{\text{geo-M}}$ still outperforms other SR-based score functions, as in \cref{Fig: SVM Toy Example Fig} \textbf{\textcolor{orange}{(a-2)}}. 

    \item[\textcolor{blue}{Case 3:}] We continue to increase the magnitudes of perturbations and \cref{Fig: SVM Toy Example Fig} \textbf{\textcolor{blue}{(b-3)}} illustrates the case where the perturbations are drawn from a uniform distribution within the interval $[-2, 2]$. Now that samples from different classes are well mixed in the 2D space, $RL_{\text{geo-M}}$ is no longer superior when the coverage level is high, as shown in \cref{Fig: SVM Toy Example Fig} \textbf{\textcolor{blue}{(a-3)}}. However, we argue that \textcolor{blue}{\textbf{Case 3}} is less concerning in practice---we probably will never consider deploying a classifier that does not work well at all before SC; see the risk achieved at coverage level $1$. Instead of relying on an SC strategy, it is more urgent to improve the base classifier in this case.
\end{description}
\textbf{Summary:}
Using the above examples, we have shown that our proposed margin-based score functions are not sensitive to the scale of the raw logits. When the base classifier is reasonable in classifying in-distribution data samples (i.e.,  achieving low risks at full coverage), margin-based scores are expected to result in good SC performance, even when test samples have low or moderate distribution shifts, as we show empirically in \cref{Sec: Experiments} below.   

\section{Experiments}
\label{Sec: Experiments}

In this section, we experiment with various multiclass classification tasks and recent DNN classifiers to verify the effectiveness of our margin-based score functions for generalized SC.

\subsection{Comparison with nontraining-based score functions using pretrained models}

\paragraph{Setups}
We take different pretrained DNN models in various classification tasks and evaluate SC performance on test datasets composed of In-D and distribution-shifted samples jointly. Specifically, our evaluation tasks include (i) \texttt{ImageNet} \citep{russakovsky2015imagenet}, the most widely used testbed for image classification, with a covariate-shifted version \texttt{ImageNet-C} \citep{hendrycks2018benchmarking} composed of synthetic perturbations, and \texttt{OpenImage-O} \citep{wang2022vim} composed of natural images similar to \texttt{ImageNet} but with disjoint labels, i.e., label-shifted samples; (ii) \texttt{iWIldCam} \citep{beery2020iwildcam} test set provides two subsets of animal images taken at different geo-locations, where one is the same as the training set serving as In-D and the other at different locations as a natural covariate-shifted version; (iii) \texttt{Amazon} \citep{ni2019justifying} test set provides two subsets of review comments by different users, producing In-D and natural covariate-shifted test samples for a language sentiment classification task; (iv) \texttt{CIFAR-10} \citep{krizhevsky2009learning}, a small image classification dataset commonly used in previous training-based SC works, together with \texttt{CIFAR-10-C} (perturbed \texttt{CIFAR-10}) and \texttt{CIFAR-100} (with disjoint labels from \texttt{CIFAR-10}),  popularly used covariate-shifted and label-shifted versions of \texttt{CIFAR-10}. \cref{Tab: model choice,Tab: dataset choice} summarize the pretrained models and datasets.

\paragraph{Confidence-score functions for comparison}
In addition to $SR_\text{max}$, $SR_\text{doctor}$ and $SR_\text{ent}$ introduced in \cref{eq: SR based SC scores} and our proposed margin-based scores $RL_{\text{geo-M}}$ and $RL_{\text{conf-M}}$ in \cref{eq: geo_raw margin}, we also consider several recent post-hoc OOD detection scores\footnote{In OOD detection, scores are usually dependent on the training data. However, these post-hoc scores can also be applied as nontraining-based SC scores as \cref{alg: confidence-based sc}, by replacing $\gD
_{\mc X}$ and $\gD_{\mc X}^{OOD}$ in \cref{alg: OOD detector} with $\gD_{\mc X, \mc Y}^{cali}$.}: (i) $RL_{\text{max}}$: the maximum raw logit~\citep{hendrycks2019scaling}; (ii) Energy: log-sum-exponential aggregation (i.e., smooth approximation to the maximum raw logit) of the raw logits~\citep{liu2020energy}; (iii) KNN: a score composed of the distances from a test data point to the $k$ nearest neighbors of the training set in the raw logit space \citep{sun2022out}; (iv) ViM---a score composed of the residual of a test sample from the principal components estimated in the feature space prior to the raw logits using training data \citep{wang2022vim}; and (v) SIRC---a composite score of the softmax response and OOD detection scores \citep{xia2022augmenting}.
\textbf{We note that} KNN, ViM, and SIRC all contain hyperparameters that are determined by the training data. To minimize the gap with our `nontraining-based' setup, we randomly sample a small number of data points\footnote{Five times the number of classes in each task from \cref{Tab: dataset choice}. We do not sample five points per class, as in practice the calibration set $\gD^{cali}_{\mc X, \mc Y}$ may be imbalanced.} from the In-D test set to tune their hyperparameters, respectively. Also, note that KNN has an additional hyperparameter $k$ that is independent of the statistics of the dataset. Empirically, we find KNN's performance is very sensitive to the choice of $k$, the task, and the classifier. Therefore, in this paper, we use $k=2$ (the empirical best) by default for KNN and provide an ablation analysis for KNN for each experiment in \cref{App: KNN scores ablation}.  

\begin{table}[!htbp]
\centering
\caption{Summary of the pretrained classifiers used for the various classification tasks}
\resizebox{1\linewidth}{!}{%
\begin{tabular}{c c c c c}
\small{\textbf{Task}}
&\small{\textbf{Model Name}}
&\small{\textbf{Source}}
&\small{\textbf{Note}}
\\
\toprule
\toprule
{} & EVA \citep{fang2023eva} &  & Top-1 acc. 88.76 \% \\
\texttt{ImageNet} & ConvNext \citep{liu2022convnet} & \texttt{timm}\tablefootnote{See \cref{APP table: timm model cards} in \cref{App: DNN classifier} for the model card information to retrieve these \texttt{timm} models.} & Top-1 acc. 86.25 \% \\
{} & VOLO \citep{yuan2022volo} & & Top-1 acc. 85.56 \% \\
{} & ResNext \citep{xie2017aggregated} &  & Top-1 acc. 85.54 \% \\
\midrule
\small{\texttt{iWildCam}} & FLYP \citep{goyal2023finetune} & Official source code\tablefootnote{\url{https://github.com/locuslab/FLYP}} & \small{Ranked $1^{st}$ on \texttt{WILDS} \citep{koh2021wilds}} \\
\midrule
\small{\texttt{Amazon}} & LISA \citep{yao2022improving} & Official source code\tablefootnote{\url{https://github.com/huaxiuyao/LISA.git}} & \small{Ranked $1^{st}$ on \texttt{WILDS}}\\
\midrule
\small{\texttt{CIFAR} \& \texttt{ImageNet}} & ScNet \citep{geifman2019selectivenet} & \small{\texttt{PyTorch re-implementation}}\tablefootnote{\url{https://github.com/gatheluck/pytorch-SelectiveNet}} & \small{Training-based SC.}\\
\bottomrule
\bottomrule
\end{tabular}}
\label{Tab: model choice}
\end{table}

\vspace{-2em}
\begin{table}[!htbp]
\centering
\caption{Summary of In-D and distribution-shifted datasets used for our SC evaluation}
\resizebox{1\linewidth}{!}{%
\begin{tabular}{lc c c cc c cc c}
\small{\textbf{Task}}
&
&\small{\textbf{In-D} (split)}
&\small{classes - samples}
&{}
&\small{\textbf{Shift-Cov}}
&\small{samples}
&{}
&\small{\textbf{Shift-Label}}
&\small{samples}
\\
\toprule
\toprule
\small{\texttt{ImageNet}}
&
&\small{\texttt{ILSVRC-2012} ('val')}
&\small{1000 - 50,000}
&{}
&\small{\texttt{ImageNet-C} (severity $3$)}
&\small{50,000 $\times$ 19}
&{}
&\small{\texttt{OpenImage-O}}
&\small{17,256}
\\
{}
&
&{}
&{}
&{}
&\small{*All types of corruptions}
&{}
&{}
&\small{

}
&{}
\\
\midrule
\small{\texttt{iWildCam}}
&
&\small{\texttt{iWildCam} ('id\_test')}
&\small{178 - 8154}
&{}
&\small{\texttt{iWildCam} ('ood\_test')}
&\small{42791}
&{}
&\small{N/A}
&\small{N/A}
\\
\midrule
\small{\texttt{Amazon}}
&
&\small{\texttt{Amazon} ('id\_test')}
&\small{5 - 46,950}
&{}
&\small{\texttt{Amazon} ('test')}
&\small{100,050}
&{}
&\small{N/A}
&\small{N/A}
\\
\midrule
\small{\texttt{CIFAR}}
&
&\small{\texttt{CIFAR-10} ('val')}
&\small{10 - 10,000}
&{}
&\small{\texttt{CIFAR-10-C} (severity $3$)}
&\small{10,000 $\times$ 19}
&{}
&\small{\texttt{CIFAR-100}}
& \small{10,000}
\\
\small{

}
&
&{}
&{}
&{}
&\small{*All types of corruptions}
&{}
&{}
&\small{}
&{}
\\
\bottomrule
\bottomrule
\end{tabular}}
\label{Tab: dataset choice}
\end{table}

\paragraph{Evaluation metrics}
We report both the RC curves and the AURC-$\alpha$ where $\alpha \in \{0.1, 0.5, 1 \}$ as discussed in \cref{Subsec: eva SC OOD}. Note that when plotting the RC curves, we omit $SR_\text{doctor}$ because it almost overlaps with $SR_{\text{max}}$, which is also observed by \citet{xia2022augmenting}.


\paragraph{Results on ImageNet}
We show in \cref{Fig: ImageNet RC curves} the RC curves of the various score functions on the pretrained model \textbf{EVA}, for different combinations of subsets of test data, as summarized in \cref{tab: Partial AURC ImageNet}. The most striking is in \cref{Fig: ImageNet RC curves}(c), which collects the results for evaluation on mixup of In-D and label-shifted samples: except for $RL_{\text{geo-M}}$, $RL_{\text{conf-M}}$ and KNN, the selection risks of other score functions do not follow a monotonic decreasing trend as coverage decreases. As coverage approaches zero, their selection risks spike up, almost to the risk level at full coverage (i.e., error rate on the whole set). This is because the other score functions do not indicate prediction confidence well in this setting and hence fail to sufficiently separate right and wrong predictions---during rejection, both right and wrong predictions are rejected indiscriminately. 
On the other hand, $RL_{\text{geo-M}}$, $RL_{\text{conf-M}}$ are better than KNN in separating correct and wrong predictions when there are no label-shifted samples, as shown in \cref{Fig: ImageNet RC curves} (a)\&(b). As a result, $RL_{\text{geo-M}}$ and $RL_{\text{conf-M}}$ have the best overall performance when In-D, covariate-shifted and label-shifted samples coexist, as shown in \cref{Fig: ImageNet RC curves} (d). Also, see \cref{tab: Partial AURC ImageNet} for numerical confirmation of the above observations, where in all cases $RL_{\text{geo-M}}$ and $RL_{\text{conf-M}}$ are the best or comparable to the best-performing  among all score functions. We present the SC results of other \texttt{ImageNet} models in \cref{App: Extra experimental results}; our margin-based score functions still stand as the best-performing among all. 

\begin{figure}[!htbp]
\centering
\resizebox{1\columnwidth}{!}{%
\begin{tabular}{c c c c c}
\centering 
&\includegraphics[width=0.24\textwidth]{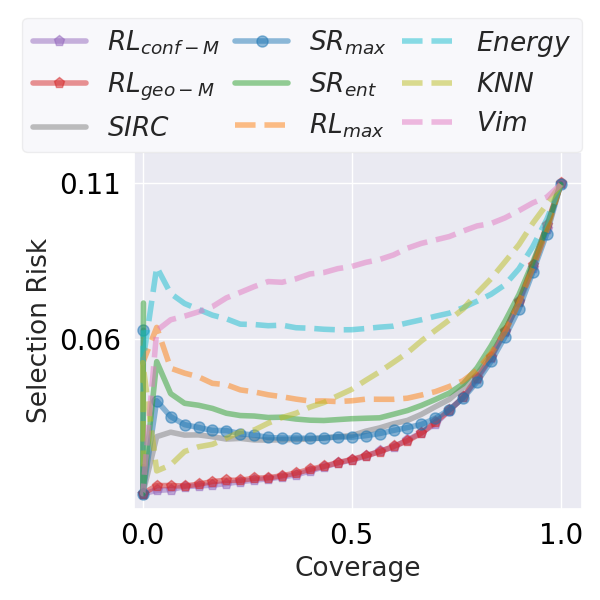}%
&\includegraphics[width=0.24\textwidth]{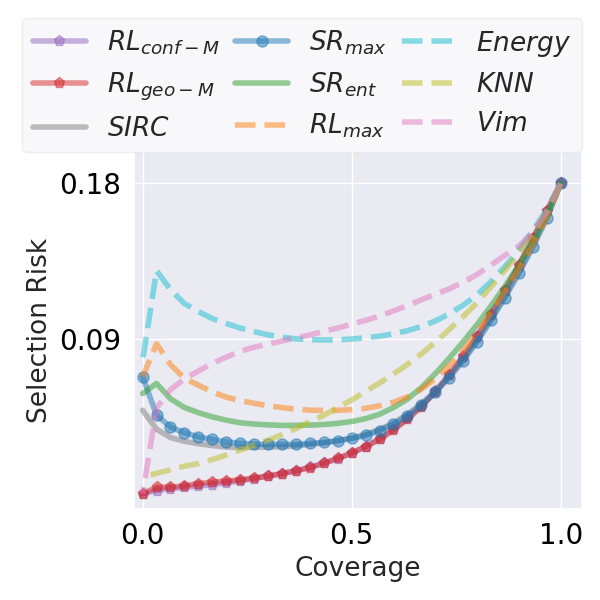}%
&\includegraphics[width=0.24\textwidth]{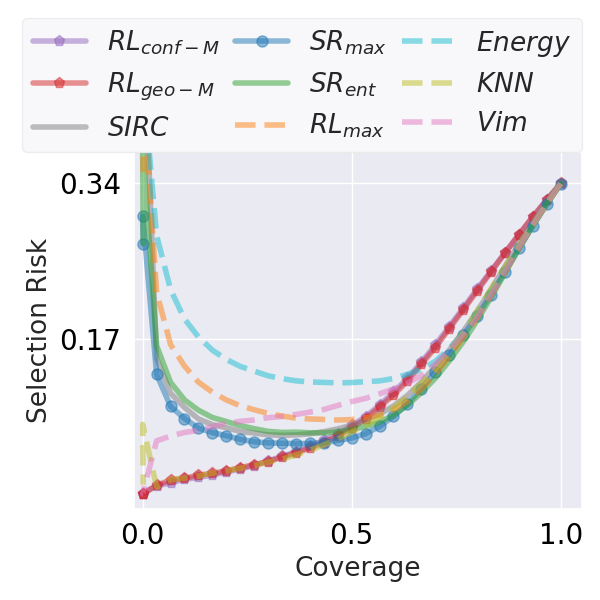}%
&\includegraphics[width=0.24\textwidth]{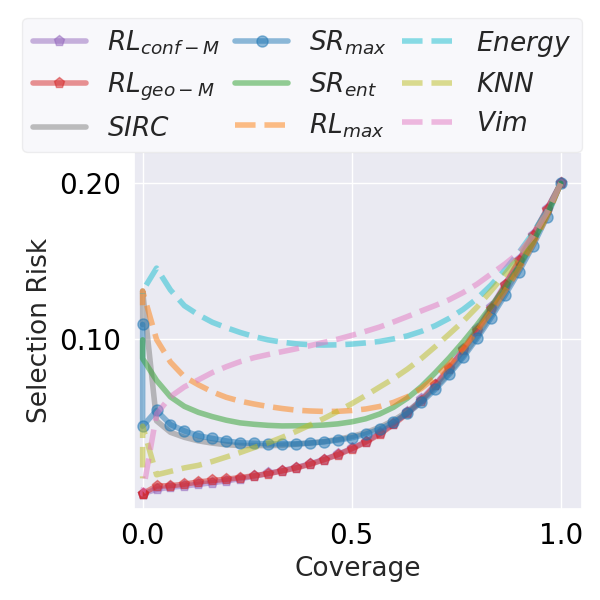}%
\\
&\small{\textbf{(a)} In-D (ImageNet)}%
&\small\small{\textbf{(b)} In-D + Shift (Cov)}%
&\small{\textbf{(c)} In-D + Shift (Label)}%
&\small{\textbf
{(d)} In-D + Shift (both)}%
\end{tabular}}%
\caption{RC curves of different confidence-score functions on the model \textbf{EVA} for ImageNet. \textbf{(a)-(d)} are RC curves evaluated using samples from \textbf{(a)} In-D samples only, \textbf{(b)} In-D and covariate-shifted samples only, \textbf{(c)} In-D and label-shifted samples only, and \textbf{(d)} all samples, respectively. We group the curves by whether they are originally proposed for SC setups (solid lines) or for OOD detection (dashed lines).}
\label{Fig: ImageNet RC curves}
\end{figure}%
\vspace{-1em}

\begin{table*}[!htbp]
\caption{Summary of AURC-$\alpha$ for \cref{Fig: ImageNet RC curves}. The AURC numbers are \emph{on the $10^{-2}$ scale---the lower, the better}. The score functions proposed for SC are highlighted in gray, and the rest are originally for OOD detection. The best AURC numbers for each coverage level are highlighted in bold, and the $2^{nd}$ and $3^{rd}$ best scores are underlined.}
\label{tab: Partial AURC ImageNet}
\centering
\resizebox{1\linewidth}{!}{%
\begin{tabular}{l ccc c ccc c ccc c ccc}
\multicolumn{1}{c}{\small{\texttt{ImageNet} - EVA}} & \multicolumn{3}
{c}{\small{In-D}} & & \multicolumn{3}
{c}{\small{In-D + Shift (Cov)}} &
 & \multicolumn{3}{c}{\small{In-D + Shift (Label)}} &
 & \multicolumn{3}{c}{\small{In-D + Shift (both)}}\\
\midrule
\multicolumn{1}{c}{\small{$\alpha$}} & 0.1 & 0.5 & 1 & & 0.1 & 0.5 & 1 & & 0.1 & 0.5 & 1 & & 0.1 & 0.5 & 1\\
\toprule
\rowcolor{Gray}
$RL_{\text{conf-M}}$ & \textbf{0.16} & \textbf{0.53} & \textbf{2.39} & & \textbf{0.24} & \textbf{0.96} & \textbf{4.77} & & \textbf{1.04} & \underline{3.34} & \underline{11.7} & & \textbf{0.34} & \textbf{1.20} & \textbf{5.43}\\
\rowcolor{Gray}
$RL_{\text{geo-M}}$ & \underline{0.27} & \underline{0.59} & \underline{2.43} & & \underline{0.37} & \underline{1.02} & \underline{4.78} & & \underline{1.20} & \underline{3.35} & \underline{11.6} & & \underline{0.48} & \underline{1.26} & \textbf{5.43}\\
\hdashline
\rowcolor{Gray}
SIRC & 2.23 & \underline{2.07} & \underline{3.36} & & 3.71 & 3.06 & \underline{5.83} & & 15.8 & 8.88 & 13.7 & & 4.61 & 3.53 & 6.52\\
\rowcolor{Gray}
$SR_{\text{max}}$ & 3.20 & 2.36 & 3.38 & & 4.52 & 3.66 & 5.93 & & 13.1 & 7.52 & 12.6 & & 5.21 & 3.75 & 6.56\\

\rowcolor{Gray}
$SR_{\text{ent}}$ & 4.28 & 3.13 & 4.04 & & 6.24 & 4.66 & 7.00 & & 16.0 & 9.19 & 13.4 & & 7.04 & 5.10 & 7.61\\
\rowcolor{Gray}
$SR_{\text{doctor}}$ & 3.22 & 2.38 & 3.40 & & 4.55 & 3.40 & 6.00 & & 13.2 & 7.55 & 12.6 & & 5.24 & 3.78 & 6.61\\

$RL_{\text{max}}$ & 5.53 & 4.05 & 4.57 & & 8.48 & 6.04 & 7.64 & & 21.1 & 11.9 & 14.9 & & 9.53 & 6.59 & 8.33\\
Energy & 8.13 & 6.60 & 6.90 & & 12.8 & 10.3 & 11.1 & & 27.3 & 16.6 & 18.1 & & 14.1 & 11.0 & 11.8\\
KNN & \underline{0.99} & 2.27 & 4.58 & & \underline{1.22} & \underline{2.89} & 6.78 & & \underline{1.18} & \textbf{3.23} & \textbf{10.8} & & \underline{1.24} & \underline{2.98} & \underline{7.16}\\
ViM & 5.48 & 7.11 & 8.31 & & 5.31 & 8.05 & 10.4 & & 5.83 & 7.89 & 13.4 & & 5.35 & 8.12 & 10.7\\
\bottomrule
\end{tabular}}
\end{table*}%

\paragraph{Results on iWildCam \& Amazon}
We report in \cref{Fig: WILDS RC curves} and \cref{tab: Partial AURC WILDS} the SC performance of different score functions on \texttt{iWildCam} and \texttt{Amazon}. Similar to the \texttt{ImageNet} experiment above, scores designed for OOD detection ($RL_{\text{max}}$, Energy, KNN and ViM) do not have satisfactory performance in SC. By contrast, existing SR-based scores ($SIRC$, $SR_{\text{max}}$, $SR_{\text{ent}}$ and $SR_{\text{doctor}}$) all demonstrate better SC potential than OOD score functions, and our margin-based score functions ($RL_{\text{conf-M}}$ and $RL_{\text{geo-M}}$) perform on par with the SR-based scores.
\begin{figure}[!tb]
\centering
\resizebox{1\columnwidth}{!}{%
\begin{tabular}{c c c | c c}
\centering
&\multicolumn{2}{c}{\texttt{iWildCam} - FYLP}
&\multicolumn{2}{c}{\texttt{Amazon} - LISA}
\\
&\includegraphics[width=0.24\textwidth]{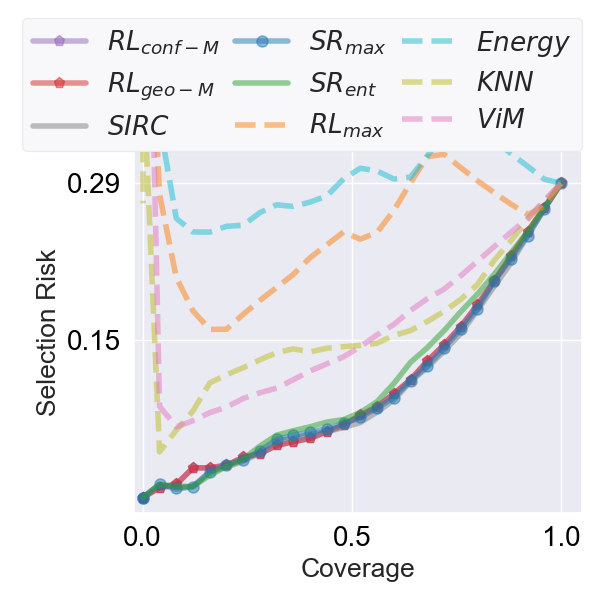}%
&\includegraphics[width=0.24\textwidth]{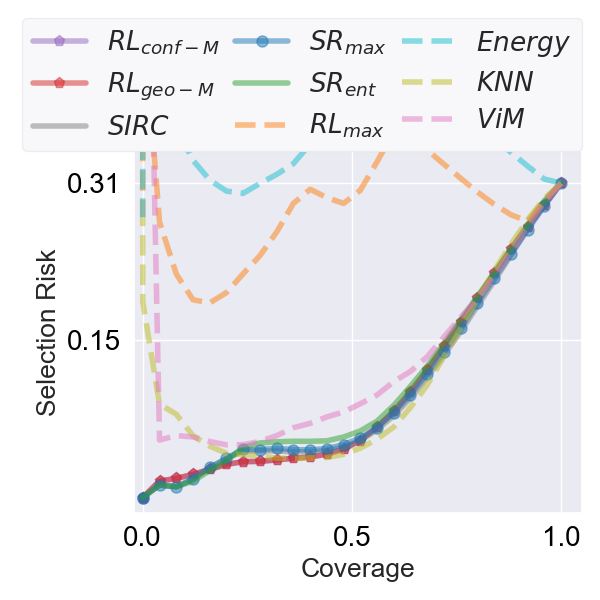}%
&\includegraphics[width=0.24\textwidth]{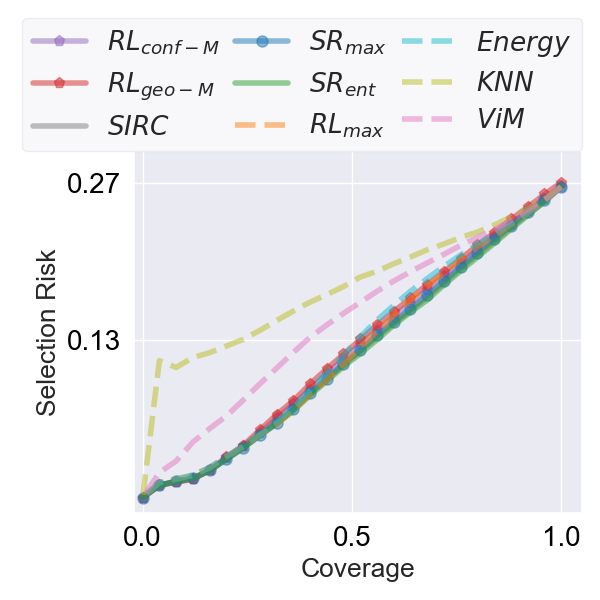}%
&\includegraphics[width=0.24\textwidth]{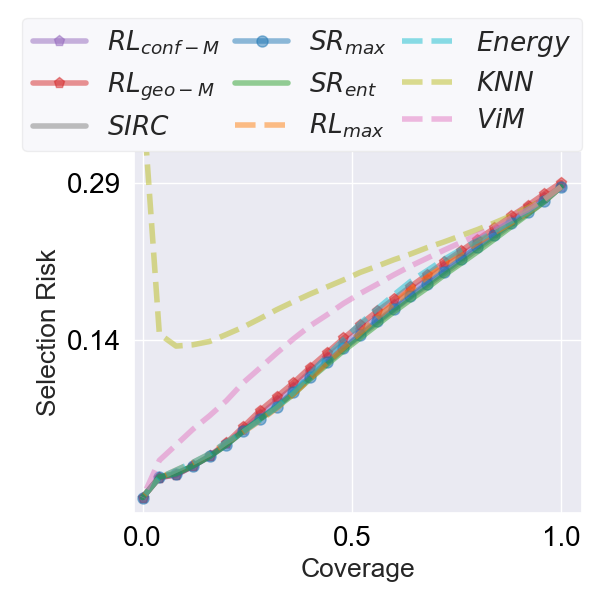}%
\\
&\small{\textbf{(a)} In-D}%
&\small{\textbf{(b)} In-D + Shift (Cov)}%
&\small{\textbf{(c)} In-D}%
&\small{\textbf{(d)} In-D + Shift (Cov)}%
\end{tabular}}%
\caption{RC curves of different confidence-score functions on the model \textbf{FLYP} for \texttt{iWildCam} and the model \textbf{LISA} for \texttt{Amazon}. \textbf{(a)\&(c)} are RC curves evaluated using In-D samples only and \textbf{(b)\&(d)} are RC curves evaluated using both In-D and covariate-shifted samples.}
\label{Fig: WILDS RC curves}
\vspace{-0.5em}
\end{figure}%
\begin{table*}[!tb]
\caption{Summary of AURC-$\alpha$ for \cref{Fig: WILDS RC curves}. The AURC numbers are \emph{on the $10^{-2}$ scale---the lower, the better}. The score functions proposed for SC are highlighted in gray, and the rest are originally for OOD detection. The best AURC numbers for each coverage level are highlighted in bold, and the $2^{nd}$ and $3^{rd}$ best scores are underlined.}
\label{tab: Partial AURC WILDS}
\centering
\resizebox{0.92\linewidth}{!}{%
\begin{tabular}{l ccc c ccc c ccc c ccc}
\multicolumn{1}{c}{} & \multicolumn{7}
{c}{\small{\texttt{iWildCam}- FYLP}} & & \multicolumn{7}{c}{\small{\texttt{Amazon} - LISA}}\\
\cline{2-8}\cline{10-16}
\vspace{-0.9em}
\\
\multicolumn{1}{c}{} & \multicolumn{3}
{c}{\small{In-D}} & & \multicolumn{3}{c}{\small{In-D + Shift (Cov)}} & & \multicolumn{3}{c}{\small{In-D}} & & \multicolumn{3}{c}{\small{In-D + Shift (Cov)}}\\
\midrule
\multicolumn{1}{c}{\small{$\alpha$}} & 0.1 & 0.5 & 1 & & 0.1 & 0.5 & 1 & & 0.1 & 0.5 & 1 & & 0.1 & 0.5 & 1\\
\toprule
\rowcolor{Gray}
$RL_{\text{conf-M}}$ & \underline{1.63} & \underline{3.88} & 10.2 & & \underline{1.84} & \textbf{3.21} & \underline{10.0}& & \textbf{1.11} & 5.31 & 12.5 & & \textbf{1.83} & 6.91 & 14.2\\
\rowcolor{Gray}
$RL_{\text{geo-M}}$ & \underline{1.63} & \underline{3.88} & \underline{10.1} & & \underline{1.84} & \textbf{3.21} & \underline{10.0} & & \underline{1.13} & 5.51 & 12.8 & & \underline{1.86} & 7.15 & 14.6\\
\hdashline
\rowcolor{Gray}
SIRC & \textbf{1.45} & \textbf{3.72} & \textbf{9.84} & & \underline{1.38} & \underline{3.5} & \textbf{9.94} & & \underline{1.14} & \underline{5.09} & \underline{12.2} & & \underline{1.88} & \underline{6.66} & \underline{13.9}\\
\rowcolor{Gray}
$SR_{\text{max}}$ & \textbf{1.45} & \underline{3.87} & \underline{10.0} & & \underline{1.38} & \underline{3.61} & \underline{10.1} & & \underline{1.14} & \underline{5.13} & \underline{12.3} & & \underline{1.88} & \underline{6.70} & \underline{14.0}\\
\rowcolor{Gray}
$SR_{\text{ent}}$ & \underline{1.46} & 4.03 & 10.6 & & \textbf{1.34} & 3.94 & 10.6 & & 1.15 & \textbf{5.06} & \textbf{12.1} & & 1.89 & \textbf{6.61} & \textbf{13.8}\\
\rowcolor{Gray}
$SR_{\text{doctor}}$ & \textbf{1.45} & \underline{3.87} & \underline{10.1} & & \underline{1.38} & 3.62 & \underline{10.1} & & \underline{1.14} & \underline{5.13} & \underline{12.2} & & \underline{1.88} & \underline{6.70} & \underline{13.9}\\

$RL_{\text{max}}$ & 29.1 & 21.4 & 24.7 & & 25.5 & 24.8 & 27.9 & & 1.26 & 5.21 & 12.5 & & 1.98 & 6.88 & 14.4\\
Energy & 35.2 & 28.3 & 29.9 & & 36.1 & 33.2 & 34.4 & & 1.26 & 5.37 & 12.8 & & 1.98 & 6.88 & 14.4\\
KNN & 6.40 & 11.1 & 15.3 & & 8.16 & 5.10 & 10.7 & & 12.1 & 14.3 & 18.2 & & 16.1 & 16.5 & 20.1\\
ViM & 13.4 & 10.7 & 15.7 & & 6.98 & 6.47 & 12.2 & & 2.33 & 8.72 & 15.0 & & 3.55 & 10.4 & 16.7\\
\bottomrule
\end{tabular}}
\end{table*}%

\subsection{Comparison with a training-based confidence-score function}
\label{Subsec: experiment result --- training based scores}
We also compare with a training-based method, ScNet~\citep{geifman2019selectivenet}. ScNet consists of a selection network and a classifier that are structurally \emph{decoupled} and trained together, allowing us to perform a faithful comparison of selection scores with a fixed classifier\footnote{We do not consider training-based score functions such as \citet{liu2019deep,huang2022self} due to the ambiguity in calculating their SR responses. During their training, a virtual class ``abstention'' is added and the softmax normalization is applied on all logits---including that of the virtual class, so it is unfair either simply dropping the abstention logit during test for score calculation or keeping the abstention logit but modifying the score calculation procedure. Retraining a classifier with the same settings but without the abstention logit is also unfair due to the requirement of a fixed classifier. Furthermore, \citet{feng2023towards} reports that the above selection methods \citep{liu2019deep,huang2022self} are not as effective as they claim.}. As shown above, score functions designed for OOD detection perform poorly for generalized SC, so here we focus on comparing our margin-based and SR-based score functions with ScNet. We first train ScNet using the training set of \texttt{CIFAR-10} and \texttt{ImageNet}, respectively; see \cref{App: ScNet details} for training details. After training, we fix both the classification and the selection heads and compute the scores and selection risks using the test setup shown in \cref{Tab: dataset choice}: (i) the ScNet selection score is taken directly from the selection head, and (ii) the margin-based and SR-based scores are computed using the classification head. 

\begin{figure}[!tb]
\centering
\resizebox{1\columnwidth}{!}{%
\begin{tabular}{c c c | c c}
\centering
&\multicolumn{2}{c}{\small{\texttt{CIFAR} - ScNet}}
&\multicolumn{2}{c}{\small{\texttt{ImageNet} - ScNet}}
\\
&\includegraphics[width=0.24\textwidth]{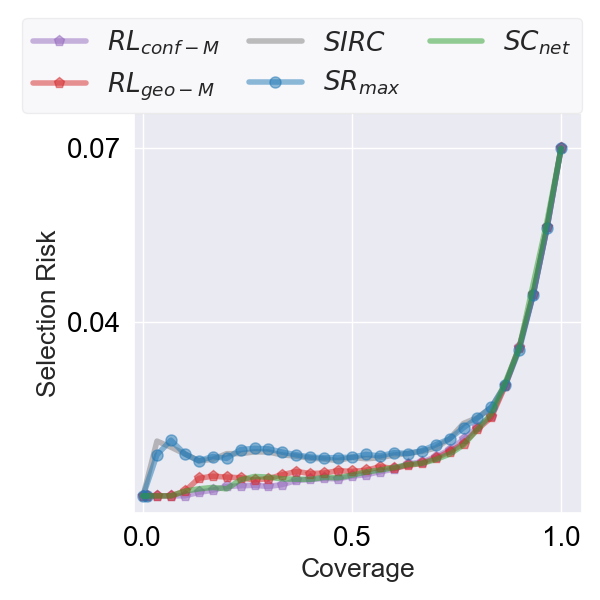}
&\includegraphics[width=0.24\textwidth]{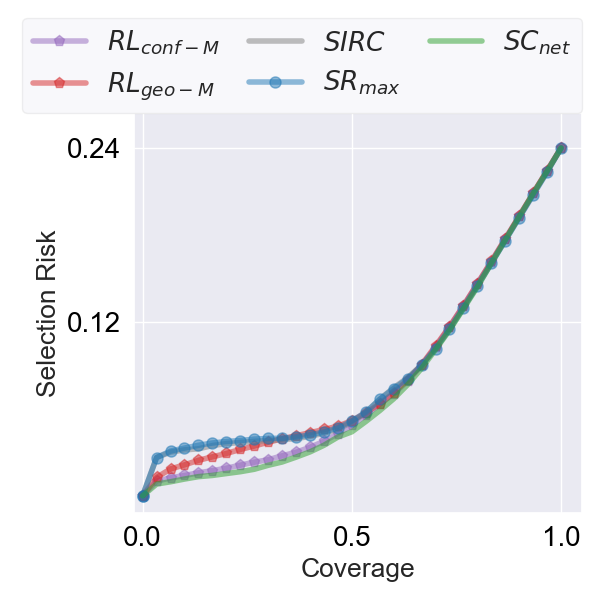}%
&\includegraphics[width=0.24\textwidth]{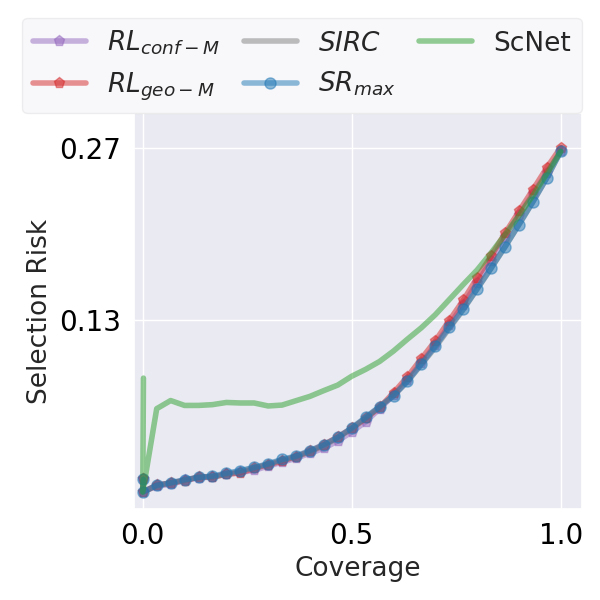}%
&\includegraphics[width=0.24\textwidth]{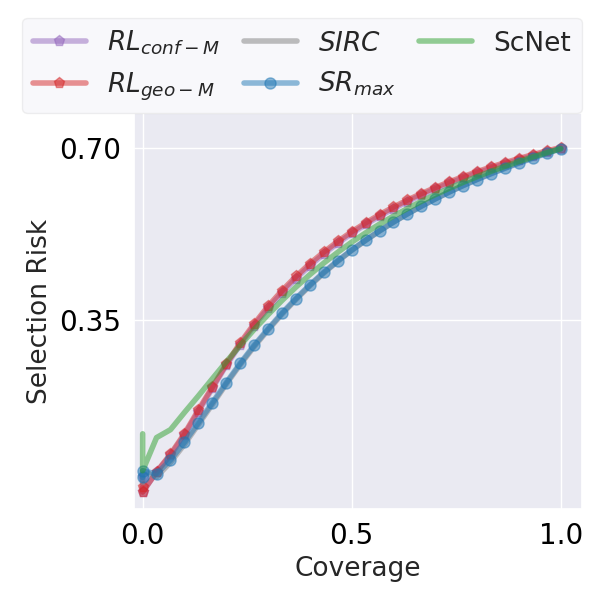}%
\\
&\small{\textbf{(a)} In-D}%
&\small{\textbf{(b)} In-D + Shift (Both)}%
&\small{\textbf{(c)} In-D}%
&\small{\textbf{(d)} In-D + Shift (Both)}%
\end{tabular}}%
\caption{RC curves of different confidence-score functions on the model \textbf{ScNet} for \texttt{CIFAR} and \texttt{ImageNet}. \textbf{(a)\&(c)} are RC curves evaluated using In-D samples only and \textbf{(b)\&(d)} are RC curves evaluated using both In-D and covariate-shifted samples.}
\label{Fig: ScNet RC curves}
\vspace{-0.5em}
\end{figure}%

\paragraph{Results}     
We show in \cref{Fig: ScNet RC curves} the RC curves achieved using ScNet, SR-based, and margin-based scores. For the \texttt{CIFAR} experiment shown in \cref{Fig: ScNet RC curves} (a)\&(b), ScNet and $RL_{\text{conf-M}}$ perform comparably and are better than $SR_{\text{max}}$ and SIRC, whereas for the \texttt{ImageNet} experiment in \cref{Fig: ScNet RC curves} (c)\&(d), $RL_{\text{conf-M}}$, $RL_{\text{geo-M}}$, $SR_{\text{max}}$ and SIRC perform comparably and are better than ScNet.\footnote{Existing training-based SC works so far have only reported SC (In-D) performance on \texttt{CIFAR-10} dataset and have not experimented with \texttt{ImageNet} using the full training set. Our results on \texttt{CIFAR-10} dataset faithfully reproduce the result originally reported in \citet{geifman2019selectivenet}.} Surprisingly, ScNet does not always lead to the best performance, even if it has access to training data. However, our margin-based scores consistently exhibit good SC performance. 

\subsection{Summary of experimental results} 
From all above experiments, we can conclude that (i) existing nontraining-based score functions for OOD detection do not perform well for generalized SC, not helping achieve reliable classification performance after rejecting low-confidence samples, and (ii) our proposed margin-based score functions $RL_{\text{geo-M}}$ and $RL_{\text{conf-M}}$ consistently perform comparably to or better than existing SR-based scores on all DL models we have tested, especially in the low-risk regime, which is of particular interest for high-stakes problems. These confirm the superiority of $RL_{\text{geo-M}}$ and $RL_{\text{conf-M}}$ as effective confidence-score functions for SC even under moderate distribution shifts for risk-sensitive applications. 

In most of our experiments, $RL_{\text{geo-M}}$ and $RL_{\text{conf-M}}$ perform similarly; only in rare cases, e.g. \cref{Fig: WILDS RC curves} (a) and \cref{Fig: ScNet RC curves} (b), $RL_{\text{conf-M}}$ slightly outperforms $RL_{\text{geo-M}}$. However, we do not think it is sufficient to conclude that $RL_{\text{conf-M}}$ is better than $RL_{\text{geo-M}}$, or vise versa. Recall how $RL_{\text{conf-M}}$ and $RL_{\text{geo-M}}$ are defined in \cref{eq: svm relative geo margin,eq: svm confidence margin} and their associated decision rules, the current practice of training DL classifiers is in favor of $RL_{\text{conf-M}}$\footnote{The cross-entropy loss is the most commonly used and minimizing it can be viewed as approximating maximizing the confidence margin. To see this, without loss of generality, assume that the magnitudes of the raw logits are ordered $z_1 > z_2 > \dots > z_K$ and that the true label of the current sample is class $1$. Then the cross-entropy loss for the current sample is $-\log \pqty{e^{z_1}/\sum_i e^{z_i}} = \log \sum_i e^{z_i - z_1} = \log \pqty{1+ \sum_{i \ge 2} e^{z_i - z_1}}$, so $\min -\log (e^{z_1}/\sum_i e^{z_i}) \equiv \min \log (1+ \sum_{i \ge 2} e^{z_i - z_1}) \equiv \min \sum_{i \ge 2} e^{z_i - z_1}$, where the last minimization problem can be approximated by $\min e^{z_2 - z_1} \equiv  \min (z_2 - z_1) \equiv  \max (z_2 - z_1)$, i.e., maximizing the confidence margin, when $e^{z_2- z_1} \gg \sum_{i \ge 3} e^{z_i - z_1}$. }. Thus, understanding the difference in behavior of $RL_{\text{geo-M}}$ and $RL_{\text{conf-M}}$ is likely to also involve investigation of the training process, which we will leave for future work. 

\section{Conclusion and discussion}
\label{Sec: conclusion and discussion}
In this paper, we have proposed \emph{generalized selective classification}, a new selective classification (SC) framework that allows distribution shifts. This is motivated by the pressing need to achieve reliable classification for real-world, risk-sensitive applications where data can come from the wild in deployment. Generalized SC \emph{covers} and \emph{unifies} existing selective classification and out-of-distribution (OOD) detection, and we have proposed two margin-based score functions for generalized SC, $RL_{\text{geo-M}}$ and $RL_{\text{conf-M}}$, which are not based on training: they are compatible for any given pretrained classifiers. Through our extensive analysis and experiments, we have shown the superiority of $RL_{\text{geo-M}}$ and $RL_{\text{conf-M}}$ over numerous recently proposed non-training-based score functions for SC and OOD detection. As the first work that touches on generalized SC, our paper can inspire several lines of future research, including at least: (i) to further improve the SC performance, one can try to align the training objective with our SC confidence-score functions here, i.e., promoting large margins; (ii) in this paper, we only consider the case where all classes are treated equally, while practical generalized SC might entail different rejection weights and costs for different classes, e.g., medical diagnosis of diseases with different levels of health implications; (iii) last but not least, finding better confidence-score functions. We hope that our small step here stimulates further research on generalized SC, bridging the widespread gaps between exploratory AI development and reliable AI deployment for practical high-stakes applications. 



\subsubsection*{Acknowledgments}
Liang H. and Sun J. are partially supported by NIH fund R01NS131314. Peng L. and Sun J. are partially supported by NIH fund R01CA287413. The authors acknowledge the Minnesota Supercomputing Institute (MSI) at the University of Minnesota for providing resources that contributed to the research results reported in this article. The content is solely the responsibility of the authors and does not necessarily represent the official views of the National Institutes of Health. This research is also part of AI-CLIMATE: ``AI Institute for Climate-Land Interactions, Mitigation, Adaptation, Tradeoffs and Economy,'' and is supported by USDA National Institute of Food and Agriculture (NIFA) and the National Science Foundation (NSF) National AI Research Institutes Competitive Award no. 2023-67021-39829.


\bibliography{main}
\bibliographystyle{tmlr}


\appendix

\section{Linear SVM and margins}
\label{App: Multiclass SVM}

We first consider binary classification. Assume training set $\set{(\mb x_i, y_i)}_{i \in [N]}$ ($[N] \doteq \set{1, \dots, N}$), where $y_i \in \set{+1, -1}$ and for notational simplicity, we assume that an extra $1$ has been appended to the original feature vectors so that we only need to consider the homogeneous form of the predictor: $f(\mb x) = \mb w^\T \mb x$. The basic idea of SVM is to maximize the worst signed geometric margin, \emph{which makes sense no matter whether the data are separable or not}: 
\begin{align} 
\label{eq:hard_binary_svm_original}
\max_{\mb w} \min_{i \in [N]}  \frac{y_i \mb w^\T \mb x_i}{\norm{\mb w}}. 
\end{align}
Note that the problem is non-convex due to the fractional form $\frac{y_i \mb w^\T \mb x_i}{\norm{\mb w}}$. Moreover, $\frac{y_i \mb w^\T \mb x_i}{\norm{\mb w}}$ is invariant to the rescaling of $\mb w$, which is bad for numerical computation (as this implies that there exist global solutions arbitrarily close to $\mb 0$ and $\infty$).

If the training set is separable, i.e., there exists a $\mb w$ such that $y_i \mb w^\T \mb x_i \ge 0, \forall\; i \in [N]$, there also exists a $\mb w$ so that $\min_{i} y_i \mb w^\T \mb x_i = 1 \; \forall\; i \in [N]$ by a simple rescaling argument. Then \cref{eq:hard_binary_svm_original} becomes 
\begin{align}
\max_{\mb w} \min_{i \in [N]}  \frac{y_i \mb w^\T \mb x_i}{\norm{\mb w}} \quad  &  \st \; \min_{i} y_i \mb w^\T \mb x_i = 1 \; \forall\; i \in [N]   \\ \Longleftrightarrow \max_{\mb w}  \frac{1}{\norm{\mb w}} \quad &  \st \; \min_{i} y_i \mb w^\T \mb x_i = 1 \; \forall\; i \in [N] \\
\Longleftrightarrow \min_{\mb w} \norm{\mb w} \quad &  \st \;  y_i \mb w^\T \mb x_i \ge 1 \; \forall\; i \in [N],  \label{eq:hard_binary_svm_cvx}
\end{align}
where \cref{eq:hard_binary_svm_cvx} is our textbook hard-margin SVM (except for the squared norm often used in the objective). A problem with \cref{eq:hard_binary_svm_cvx} is that the constraint set is infeasible for inseparable training data. To fix this issue, we can allow slight violations in the constraint and penalize these violations in the objective of \cref{eq:hard_binary_svm_cvx}, arriving at 
\begin{align}
    \min_{\mb w} \norm{\mb w}^2 + C\sum_{i \in [N]} \xi_i \quad &  \st \;  y_i \mb w^\T \mb x_i \ge 1 - \xi_i, \xi_i \ge 0 \; \forall\; i \in [N], 
\end{align}
which is our textbook soft-margin SVM.

Now for multiclass classification, let us assume the data space: $\mc X \times \mc Y = \R^d \times \set{1, \dots, K}$ with $K \ge 3$. The classifier takes the form $f(\mb x) = \mb W^\T \mb x$, where $\mb W \in \R^{d \times K}$. We note that from binary SVM, people create the notion of \emph{confidence margin}:
\begin{align}
\mathrm{ConfMargin}(\mb x_i, \mb w) \doteq y_i \mb w^\T \mb x_i, 
\end{align}
which for the binary case is simply the signed geometric margin rescaled by $\norm{\mb w}$. The standard multiclass decision rule is\footnote{The decision rule for the binary case is $\argmax_{y \in \set{+1, -1}} \; y\mb w^\T \mb x$. Therefore, we do not need to worry about the $\mb w$'s scaling. }
\begin{align} 
\label{eq:multiclass_standard_dr}
\argmax_{j \in [K]}  \mb w_j^\T \mb x, 
\end{align}
where $\mb w_j$ is the $j$-th column of $\mb W$. To correctly classify all points, we need 
\begin{align}
   \forall\; i \in [N], \,  y_i = \argmax_{j \in [K]}  \mb w_j^\T \mb x  \Longleftrightarrow \forall\; i \in [N],\; \mb w_{y_i}^\T \mb x_i > \max_{y \in \mc Y \setminus \set{y_i}} \mb w_{y}^\T \mb x_i. 
\end{align}
This motivates the multiclass hard-margin SVM, separability assumed: 
\begin{align}
    \min_{\mb W}  \sum_{j \in [K]} \norm{\mb w_j}^2  \quad \st \; \mb w_{y_i}^\T \mb x_i - \max_{y \in \mc Y \setminus \set{y_i}} \mb w_{y}^\T \mb x_i \ge 1, \; \forall\;  i \in [N], 
\end{align}
where terms $w_{y_i}^\T \mb x_i - \max_{y \in \mc Y \setminus \set{y_i}} \mb w_{y}^\T \mb x_i$ can be viewed as \emph{multiclass confidence margins}, natural generalizations of confidence margins for the binary case. The corresponding soft-margin version is 
\begin{align}
    \min_{\mb W} \;  \sum_{j \in [K]} \norm{\mb w_j}^2 + C \sum_{i \in [N]} \xi_i \; 
    \st \; \mb w_{y_i}^\T \mb x_i - \max_{y \in \mc Y \setminus \set{y_i}} \mb w_{y}^\T \mb x_i \ge 1 - \xi_i, \xi_i \ge 0 \; \forall\;  i \in [N]. 
\end{align}
Both hard- and soft-margin versions are convex and thus more convenient for numerical optimization. 

On the other hand, if we strictly follow the geometric margin interpretation, it seems more natural to formulate multiclass SVM as follows. Consider the decision rule:
\begin{align} 
\label{eq:multiclass_geo_dr}
\argmax_{j \in [K]}  \frac{\mb w_j^\T \mb x}{\norm{\mb w_j}}, 
\end{align}
which would classify all points correctly provided that there exists a $\mb W \in \R^{d \times K}$ satisfying 
\begin{align}
\forall\; i \in [N], \; \frac{\mb w_{y_i}^\T \mb x_i}{\norm{\mb w_{y_i}}} > \max_{y \in \mc Y \setminus \set{y_i}}\frac{\mb w_{y}^\T \mb x_i}{\norm{\mb w_{y}}}. 
\end{align}
This motivates an optimization problem on the worst \emph{geometric margins}: 
\begin{align}  
\label{eq:multiclass_rel_geo_margin_v1}
    \max_{\mb W} \min_{i \in [N]} \paren{\frac{\mb w_{y_i}^\T \mb x_i}{\norm{\mb w_{y_i}}} - \max_{y \in \mc Y \setminus \set{y_i}}\frac{\mb w_{y}^\T \mb x_i}{\norm{\mb w_{y}}}} \text{.}
\end{align}
However, this problem is non-convex and thus not popularly adopted. 

\section{Asymptotic behaviors of $SR_{\text{max}}$, $SR_{\text{doctor}}$, and $SR_{\text{ent}}$}
\label{App: sr-based scores limits}
Recall from mathematical analysis that two functions $f(x)$ and $g(x)$ are \emph{asymptotically equivalent} as $x \to \infty$, written as $f(x) \sim g(x)$ as $x \to \infty$, if and only if $f(x) = g(x) (1+ o(1)) \; \text{as} \; x \to \infty$, where $o(\cdot)$ is the standard small-o notation. Note that $f(x) \sim g(x) \Longleftrightarrow g(x) \sim f(x)$. 
\begin{lemma}
    Consider the raw logits $\mb z$, and without loss of generality assume that they are ordered in descending order without any ties, i.e., $z^{(1)} > z^{(2)} > \cdots $. We have that as $\lambda \to \infty$, 
    \begin{align*}
        SR_{\mathrm{max}} (\lambda\mb z) \sim \exp(- e^{\lambda (z^{(2)}- z^{(1)})}), \quad  SR_{\mathrm{doctor}} (\lambda\mb z) \sim 1- \exp\paren{2 e^{\lambda (z^{(2)} - z^{(1)} )}},  \quad SR_{\mathrm{ent}}(\lambda\mb z) \sim - e^{\lambda(z^{(2)} - z^{(1)})}.   
    \end{align*}
    Moreover, all of the asymptotic functions are monotonically increasing with respect to $z^{(1)} - z^{(2)}$. 
\end{lemma}
\begin{proof}
First, for $SR_{\text{max}}$, we have 
\begin{align}
    \log SR_{\text{max}} (\lambda \mb z) 
    = \log \frac{e^{\lambda z^{(1)}}}{\sum_i e^{\lambda z^{(i)}}} 
    =  -\log \sum_i e^{\lambda (z^{(i)}- z^{(1)})} 
    \sim  -\log (1 +  e^{\lambda (z^{(2)}- z^{(1)})}) 
\end{align}
as $\lambda \to \infty$, because $\sum_{i \ge 3} e^{\lambda (z^{(i)}- z^{(1)})}/(1 +  e^{\lambda (z^{(2)}- z^{(1)})}) \to 0$ as $\lambda \to \infty$. Moreover, as $\lambda \to \infty$, 
\begin{align}
    e^{\lambda (z^{(2)}- z^{(1)})} \to 0 
    \Longrightarrow -\log (1 +  e^{\lambda (z^{(2)}- z^{(1)})}) \sim  - e^{\lambda (z^{(2)}- z^{(1)})},  
\end{align} 
as $\log(1+x) \sim x$ when $x \to 0$. So we conclude that 
\begin{align}
    SR_{\text{max}} (\lambda\mb z) \sim \exp\paren{- e^{\lambda (z^{(2)}- z^{(1)})}}  \; \text{as} \; \lambda \to \infty. 
\end{align}

Now consider $SR_{\text{doctor}}$. Applying a similar argument as above, we have 
\begin{align}
   \log \norm{\sigma(\lambda\mb z)}_2^2 
   & = \log \sum_{i} \frac{e^{2\lambda z^{(i)}}}{(\sum_j e^{\lambda z^{(j)}})^2}
   =  \log \sum_{i} \frac{e^{2\lambda (z^{(i)} - z^{(1)})}}{(\sum_j e^{\lambda (z^{(j)} - z^{(1)} )})^2} \nonumber \\
   & = -2 \log \sum_j e^{\lambda (z^{(j)} - z^{(1)} )} + \log \sum_i e^{2\lambda (z^{(i)} - z^{(1)})}  \\ 
   & \sim -2 \log \paren{1+ e^{\lambda (z^{(2)} - z^{(1)} )}} + \log \paren{1+ e^{2\lambda (z^{(2)} - z^{(1)})} } \\
   & \sim -2 e^{\lambda (z^{(2)} - z^{(1)} )} +  e^{2\lambda (z^{(2)} - z^{(1)})} \\
   \label{eq: last SR doctor}
   & \sim -2 e^{\lambda (z^{(2)} - z^{(1)} )}
\end{align}
as $\lambda \to \infty$, where \cref{eq: last SR doctor} holds as $e^{2\lambda (z^{(2)} - z^{(1)})}$ is lower order than $-2 e^{\lambda (z^{(2)} - z^{(1)} )}$ when $z^{(2)} - z^{(1)} < 0$ so that $e^{\lambda (z^{(2)} - z^{(1)} )} < 1$. Therefore, as $\lambda \to \infty$, 
\begin{align}
    SR_{\text{doctor}} (\lambda\mb z) = 1 - \norm{\sigma(\lambda\mb z)}_2^{-2}   \sim 1- \exp\paren{2 e^{\lambda (z^{(2)} - z^{(1)} )}}. 
\end{align}
Finally, for $SR_{\text{ent}}$, we have that when $\lambda \to \infty$, 
\begin{align}
    SR_{\text{ent}}(\lambda\mb z) 
    & = \sum_i \frac{e^{\lambda z^{(i)}}}{\sum_{j} e^{\lambda z^{(j)}}} \log \frac{e^{\lambda z^{(i)}}}{\sum_{j} e^{\lambda z^{(j)}}} 
    = \sum_i \frac{e^{\lambda (z^{(i)} -z^{(1)} )}}{\sum_{j} e^{\lambda (z^{(j)} - z^{(1)})}} \log \frac{e^{\lambda (z^{(i)} - z^{(1)})}}{\sum_{j} e^{\lambda (z^{(j)} - z^{(1)})}} \\
    & = \frac{1}{\sum_{j} e^{\lambda (z^{(j)} - z^{(1)})}} \sum_i e^{\lambda (z^{(i)} -z^{(1)} )} \paren{\lambda (z^{(i)} - z^{(1)}) - \log \sum_{j} e^{\lambda (z^{(j)} - z^{(1)})}} \\
    & \sim \frac{1}{\sum_{j} e^{\lambda (z^{(j)} - z^{(1)})}} \sum_i \brac{e^{\lambda (z^{(i)} -z^{(1)} )} \lambda (z^{(i)} - z^{(1)})}  \\
    & \hspace{10em} (\text{as $\log \sum_{j} e^{\lambda (z^{(j)} - z^{(1)})}/(\lambda (z^{(i)} - z^{(1)})) \in o(1)$ when $\lambda \to \infty$}) \nonumber \\
    \label{eq: ent_last}
    & \sim \frac{1}{\sum_{j} e^{\lambda (z^{(j)} - z^{(1)})}} \brac{e^{\lambda (z^{(2)} -z^{(1)} )} \lambda (z^{(2)} - z^{(1)})}, 
\end{align} 
where \cref{eq: ent_last} holds because $\sum_{i \ge 3} e^{\lambda (z^{(i)} -z^{(1)} )} \lambda (z^{(i)} - z^{(1)})/(e^{\lambda (z^{(2)} -z^{(1)} )} \lambda (z^{(2)} - z^{(1)})) = \sum_{i \ge 3} e^{\lambda (z^{(i)} -z^{(2)} )} \paren{z^{(i)} - z^{(1)}}/\paren{z^{(2)} - z^{(1)}} \in o(1)$ as $\lambda \to \infty$. Continuing the above argument, we further have that as $\lambda \to \infty$, 
\begin{align}
    \log(- SR_{\text{ent}}(\lambda\mb z)) 
    & \sim - \log  \sum_{j} e^{\lambda (z^{(j)} - z^{(1)})} + \lambda(z^{(2)} - z^{(1)}) + \log (-\lambda(z^{(2)} - z^{(1)})).
    \label{eq: something}
\end{align}
\hy{Let's write $x \triangleq -(z^{(2)} - z^{(1)})$. The last two terms in \cref{eq: something} can be re-written as $-x + \log (x)$. Since $\lim_{x \to \infty}\frac{\log (x)}{x} = 0$, we thus have $-x + \log (x) = -x (1 + o(1))$ as $x \to \infty$, and hense $-x + \log (x) \sim -x$ by the definition of the asymptotic equivalence. Therefore, we have:}
\begin{align}
    \log(- SR_{\text{ent}}(\lambda\mb z))
    & \sim - \log  (1 +  e^{\lambda (z^{(2)} - z^{(1)})}) + \lambda(z^{(2)} - z^{(1)})\\
    & \sim - e^{\lambda (z^{(2)} - z^{(1)})} + \lambda(z^{(2)} - z^{(1)}) \sim \lambda(z^{(2)} - z^{(1)}).  
\end{align}

So we conclude that 
\begin{align}
    SR_{\text{ent}}(\lambda\mb z) \sim - \exp\paren{\lambda(z^{(2)} - z^{(1)})} \quad \; \text{as} \; \lambda \to \infty,  
\end{align}
completing the proof.  
\end{proof}

\section{Evaluation metrics for OOD detection vs. evaluation metrics for generalized SC}
\label{App: eva OOD}
\begin{wraptable}{r}{0.4\textwidth}
    \vspace{-1em}
    \centering 
    \caption{Evaluation of $s_{1}$ and $s_{2}$ using popular OOD metrics. The better numbers are highlighted in bold.}
    \label{tab: teaser OOD evaluation}
    \begin{tabular}{lcc}
     & & \\
    \multicolumn{1}{l}{OOD metric} & \multicolumn{1}{c}{$s_1$} & \multicolumn{1}{c}{$s_2$} \\
    \toprule
    AUROC ($\uparrow$) & 0.765 & \textbf{0.944}\\
    AUPR ($\uparrow$) & 0.987 & \textbf{0.997}\\
    FPR@TPR=0.95 ($\downarrow$) & 0.816 & \textbf{0.279}\\
    \bottomrule
    \end{tabular}
    \vspace{-1em}
\end{wraptable}
\hy{The commonly used evaluation metrics for OOD detection do not reflect the classification performance~\citep{franc2023reject}. Here we provide a quantitative supporting example, in comparison with the RC curve for generalized SC.}

OOD (mostly label-shift) detection as formulated in \cref{eq: def of ood detection} can be viewed as a binary classification problem: selected and rejected samples form the two classes. So pioneer work on OOD detection, such as \citet{hendrycks2016baseline}, proposes to evaluate OOD detection in a manner similar to that of binary classification, e.g., using the Area Under the Receiver Operating Characteristic (AUROC) curve~\citep{davis2006relationship} and Area Under the Precision-Recall curve (AUPR)~\citep{saito2015precision} to measure the separability of In-D and OOD samples.\footnote{A single-point metric, False Positive Rate (FPR) at $0.95$ True Positive Rate (TPR), is also popularly used as a companion~\citep{liang2017enhancing, wang2022vim, liu2020energy, djurisic2022extremely, sun2022out, yang2022openood}.} 
However, two important aspects are missing in OOD detection, and hence also its performance evaluation, if we are to focus on the performance on the accepted samples:
\begin{enumerate}[leftmargin=1em,nosep]
    \item Pretrained classifiers do not always make wrong predictions on label-shifted samples, and hence these OOD samples should not be blindly rejected; 
    \item In-D samples that might have been correctly classified can be rejected due to poor separation of In-D and OOD samples, leading to worse classification performance on the selected part.
\end{enumerate} 

To demonstrate our points quantitatively, we take the pretrained model \textbf{EVA}\footnote{See \cref{App: DNN classifier} for model card information. This model is also used in the experiments of \cref{Sec: Experiments}.} from \texttt{timm}~\citep{rw2019timm} that achieves $> 88\%$ top 1 accuracy on the ImageNet validation set. We then mix \texttt{ImageNet} validation set (In-D samples) with \texttt{ImageNet-O} (OOD samples, label shifted)~\citep{hendrycks2018benchmarking}, and evaluate two score functions $s_1$ and $s_2$\footnote{$s_1$ is our proposed $RL_\text{conf-M}$ and $s_2$ is ViM.} using both generalized SC formulation (via RC curves) and OOD detection (via AUROC and AUPR). 

%
\begin{figure}[!htbp]
\centering 
\vspace{-1em}
\resizebox{1\linewidth}{!}{%
\begin{tabular}{c c c}
\centering
\includegraphics[width=0.3\textwidth]{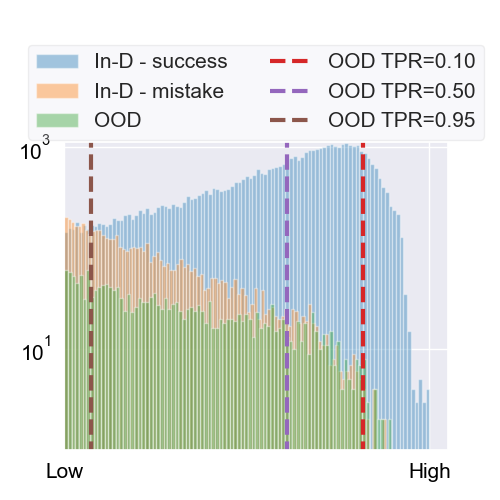}%
&\includegraphics[width=0.3\textwidth]{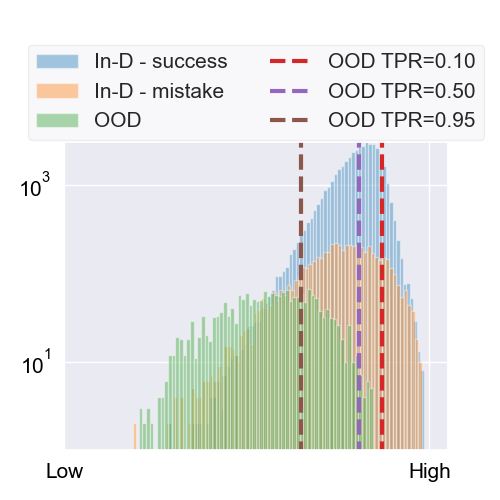}%
&\includegraphics[width=0.3\textwidth]{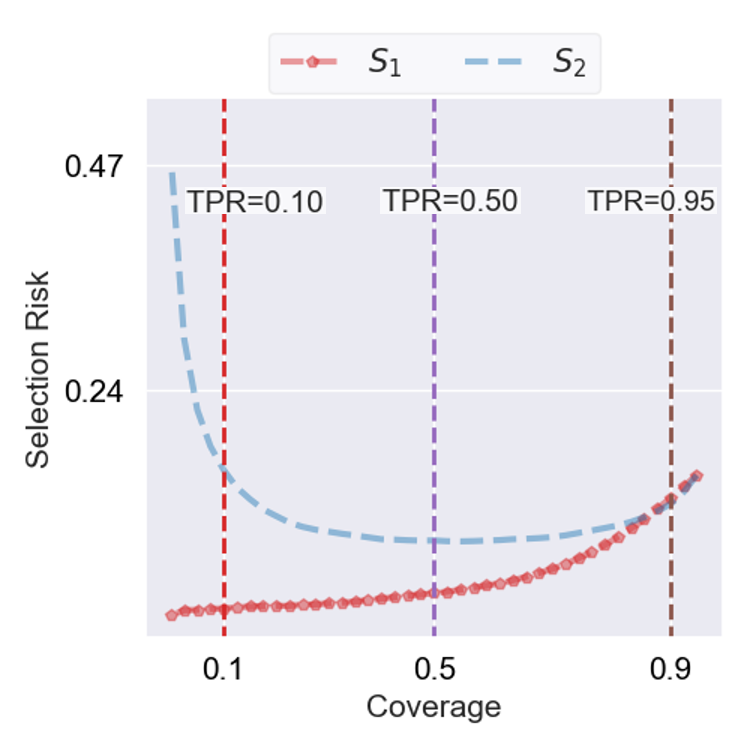}%

\\
\small{\textbf{(a)} $s_1$ score distributions}
&\small{\textbf{(b)} $s_2$ score distributions}
&\small{\textbf{(c)} RC curves}
\end{tabular}}%
\caption{Score distributions of $s_1$ and $s_2$ (a)-(b) and their RC curves (c). In (a) and (b), In-D samples that are \emph{correctly} classified by \textbf{EVA} are shown in blue, while In-D samples that are \emph{incorrectly} classified are shown in orange; OOD samples (label-shifted) are shown in green. The vertical dashed lines in (a)-(c) corresponds to different True-Positive-Rate cutoffs in the AUROC metric in OOD detection.}%
\label{Fig: teaser OOD and SC}
\vspace{-1em}
\end{figure}

According to \cref{tab: teaser OOD evaluation}, $s_2$ is considered superior to $s_1$ by all metrics for OOD detection. Correspondingly, from \cref{Fig: teaser OOD and SC}(a) and (b), we observe that the scores of the label-shifted samples (green) and those of the In-D samples (blue and orange) are more separated by $s_2$ than by $s_1$. However, we can also quickly notice one issue: In-D samples are not completely separated from OOD samples---a threshold intended to reject label-shifted samples will inevitably reject a portion of In-D samples at the same time, even though a large portion of In-D samples have been correctly classified (blue); In-D samples that can be correctly classified (blue) are less separated from those misclassified ones (orange) by $s_2$ than by $s_1$. This problem cannot be revealed by the OOD metrics in \cref{tab: teaser OOD evaluation}, but is captured by the RC curves in \cref{Fig: teaser OOD and SC}(c) where the selection risk of $s_2$ (blue) increases as more OOD samples are rejected (TPR from $0.95$ to $0.1$ as indicated by the vertical dashed lines). In contrast, the more samples rejected by $s_1$ (smaller coverage), the lower the selection risk, implying that $s_1$ serves SC better.

\section{Rejection patterns of different score functions}
\label{App: SVM example}
We plot in \cref{Fig: App - Rejection Heatmaps} the heatmap of the score values for each score function. During SC, samples located in the darker areas (with low score values) will be rejected before those located in the brighter areas (with high score values).
\begin{figure}[!htbp]
\centering
\vspace{-1em}
\resizebox{1\columnwidth}{!}{%
\begin{tabular}{c c c c c}
\centering
\includegraphics[width=0.19\textwidth]{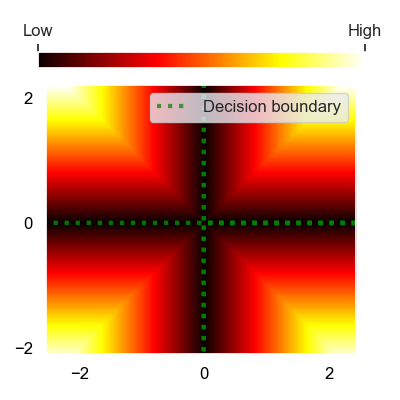}
&\includegraphics[width=0.19\textwidth]{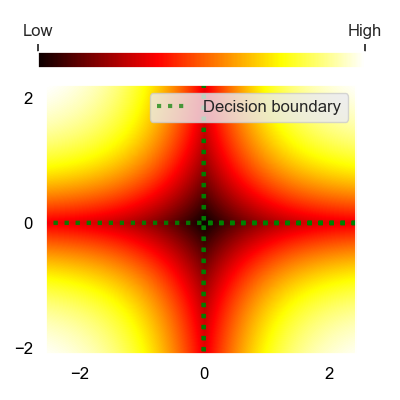}
&\includegraphics[width=0.19\textwidth]{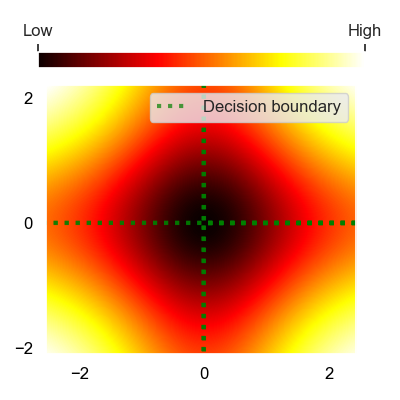}%
&\includegraphics[width=0.19\textwidth]{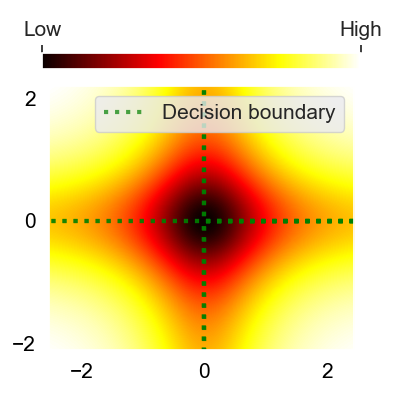}%
&\includegraphics[width=0.19\textwidth]{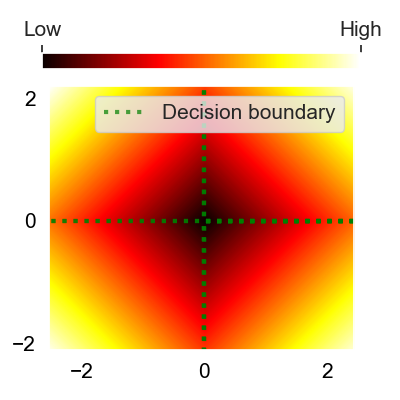}%
\\
\small{\textbf{(a)} $RL_{\text{geo-M}}$}%
&\small{\textbf{(a)} $SR_{\text{max}}$}%
&\small{\textbf{(b)} $SR_{\text{ent}}$}%
&\small{\textbf{(c)} $SR_{\text{doctor}}$}%
&\small{\textbf{(d)} $RL_{\text{max}}$}%
\end{tabular}}%
\caption{Heatmaps of rejection patterns (distribution of scores). Note that because we rescale the scores for good visualization, the colors are not cross-comparable between different score functions.}
\label{Fig: App - Rejection Heatmaps}
\vspace{-0.5em}
\end{figure}%

\section{Timm model cards}
\label{App: DNN classifier}
\begin{table}[!htbp]
\caption{Names of model cards in library \texttt{timm} to retrieve the models for  \texttt{ImageNet}}
\centering
\resizebox{0.9\linewidth}{!}{%
\begin{tabular}{c c c c c}
\label{APP table: timm model cards}
\textbf{Dataset}
&{}
&\textbf{Model name}
&\textbf{Model card name}
&\textbf{Top-1 Acc. ($\%$
)}
\\
\toprule
\toprule
{}
&{}
&{EVA (ViT)}
&{eva\_giant\_patch14\_224.clip\_ft\_in1k}
&{88.76}
\\
\cline{3-5}
\vspace{-0.8em}
\\
\texttt{ImageNet}
&{}
&{ConvNext}
&{convnextv2\_base.fcmae\_ft\_in22k\_in1k}
&{86.25}
\\
\cline{3-5}
\vspace{-0.8em}
\\
{}
&{}
&{VOLO}
&{volo\_d4\_224.sail\_in1k}
&{85.56}
\\
\cline{3-5}
\vspace{-0.8em}
\\

&{}
&{
ResNext}

&{seresnextaa101d\_32x8d.sw\_in12k\_ft\_in1k}
&{85.94}
\\

\bottomrule
\bottomrule
\end{tabular}}
\end{table}
\cref{APP table: timm model cards} shows the names of the model cards used to retrieve the pretrained models for \texttt{ImageNet} from the \texttt{timm} library. Our considerations for choosing these models are as follows: \textbf{(i)} the models should cover a wide range of recent and popular architectures, and \textbf{(ii)} they should achieve high top-$1$ accuracy to represent recent advances of image classification. 

\section{Training details for ScNet}
\label{App: ScNet details}
We use the unofficial \texttt{PyTorch} implementation\footnote{\url{https://github.com/gatheluck/pytorch-SelectiveNet}} of the original SelectiveNet \citep{geifman2019selectivenet} due to the out-of-date \texttt{Keras} environment of the original repository\footnote{\url{https://github.com/anonygit32/SelectiveNet}}. The \texttt{PyTorch} implementation follows the training method proposed in \citet{geifman2019selectivenet} and faithfully reproduces the results of \texttt{CIFAR-10} experiment reported in the original paper. We add the \texttt{ImageNet} experiment on top of the \texttt{PyTorch} code, as it is not included in the original code or the paper. \cref{Tab: ScNet Trainig detail} summarizes the key hyperparameters to produce the results reported in this paper.
\begin{table}[!htbp]
\centering
\caption{Key hyperparameters for the ScNet training used in this paper}
\resizebox{1\linewidth}{!}{%
\begin{tabular}{c c c c c c c c }
\small{\textbf{Dataset}}
&\small{\textbf{Model architecture}}
&\small{\textbf{Dropout prob.}}
&\small{\textbf{Target coverage}}
&\small{\textbf{Batch size}}
&\small{\textbf{Total epochs}}
&\small{\textbf{Lr (base)}}
&\small{\textbf{Scheduler}}
\\
\toprule
\toprule
\small{\texttt{CIFAR-10}} & \small{VGG} & \small{0.3} & \small{0.7} &\small{128} & \small{300} & \small{0.1} & \small{StepLR}\\
\midrule
\small{\texttt{ImaegNet-1k}} & \small{resnet34} & \small{N/A} & \small{0.7} & \small{768} & \small{250} & \small{0.1} & \small{CosineAnnealingLR}\\
\bottomrule
\bottomrule
\end{tabular}}
\label{Tab: ScNet Trainig detail}
\end{table}

\section{Additional \texttt{ImageNet} experiments}
\label{App: Extra experimental results}
We report in \cref{APPFig: ImageNet RC curves} the RC curves of different score functions on models \texttt{ConvNext}, \texttt{ResNext}, and \texttt{VOLO} for \texttt{ImageNet}, and summarize their AURC statistics in \cref{app-tab: Partial AURC ImageNet}. 

\begin{figure}[!tb]
\centering
\resizebox{1\columnwidth}{!}{%
\begin{tabular}{c c c c c}
\centering
&\multicolumn{4}{c}{\texttt{ImageNet} - CovNext}
\\
&\includegraphics[width=0.24\textwidth]{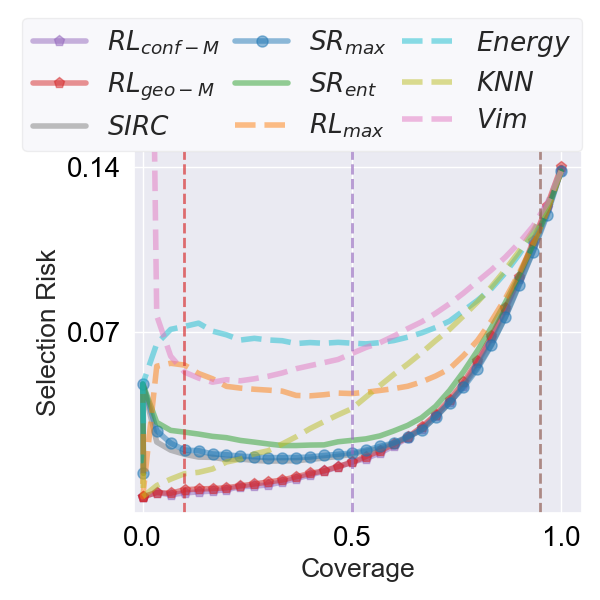}%
&\includegraphics[width=0.24\textwidth]{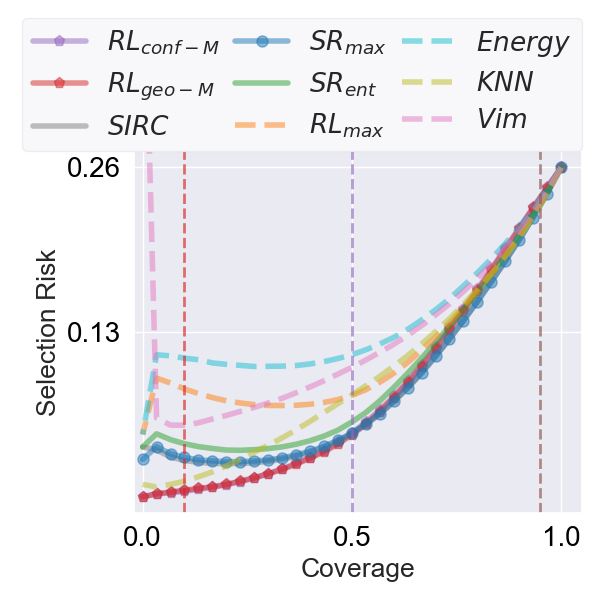}%
&\includegraphics[width=0.24\textwidth]{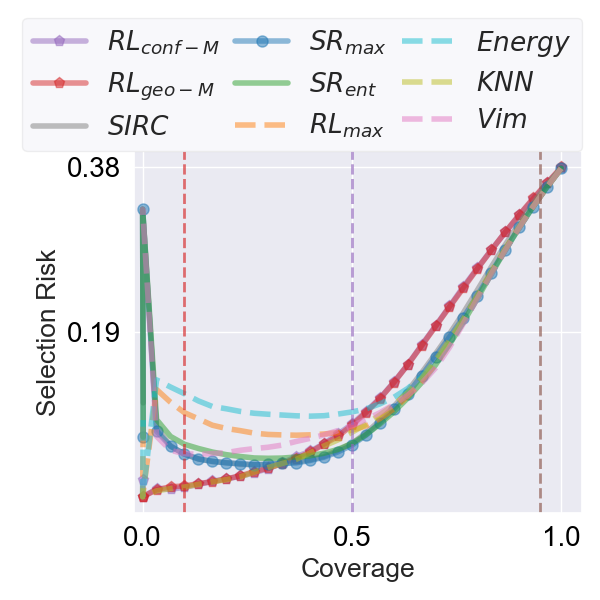}%
&\includegraphics[width=0.24\textwidth]{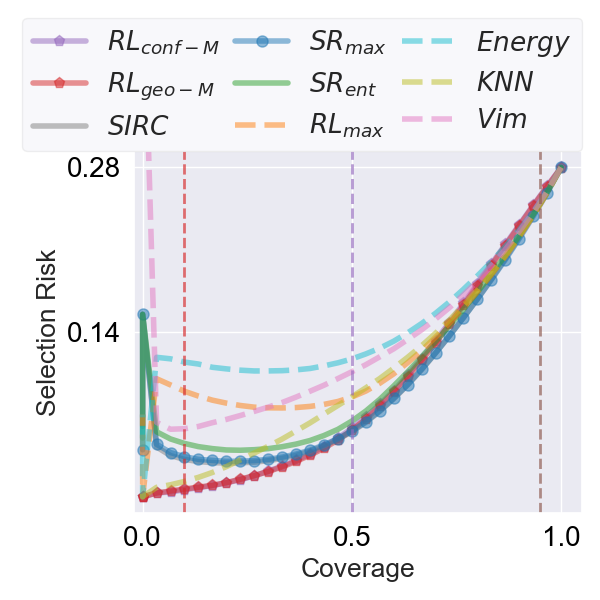}%
\\
&\multicolumn{4}{c}{\texttt{ImageNet} - ResNext}
\\
&\includegraphics[width=0.24\textwidth]{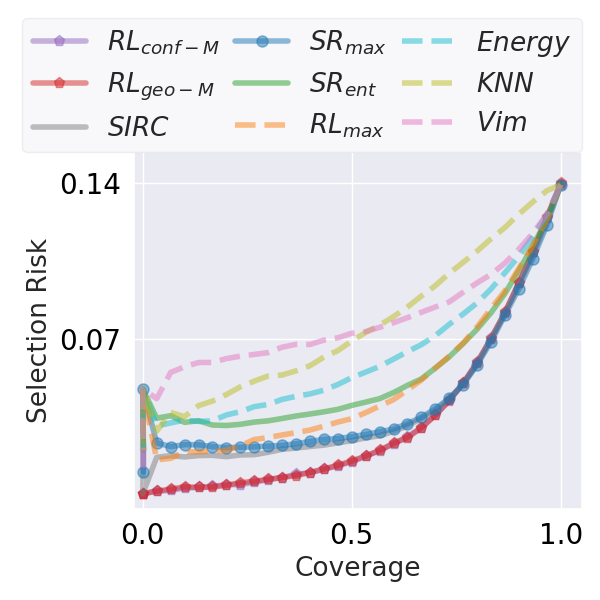}%
&\includegraphics[width=0.24\textwidth]{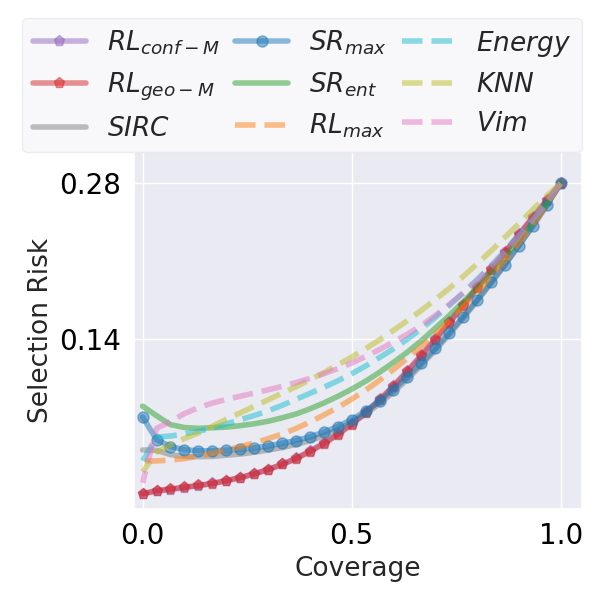}%
&\includegraphics[width=0.24\textwidth]{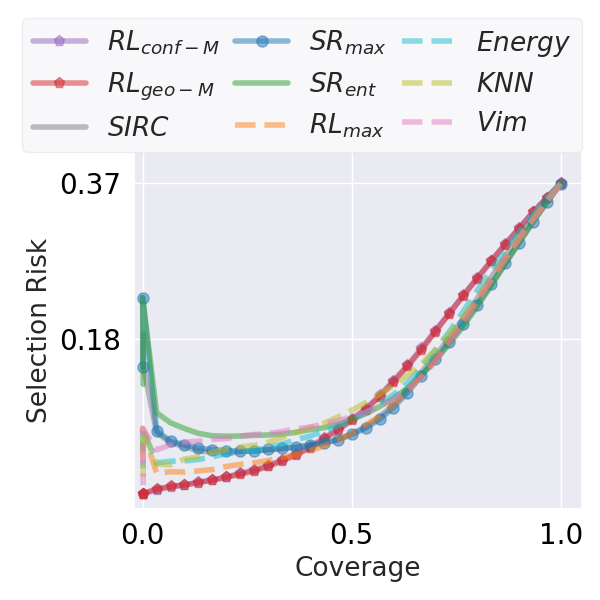}%
&\includegraphics[width=0.24\textwidth]{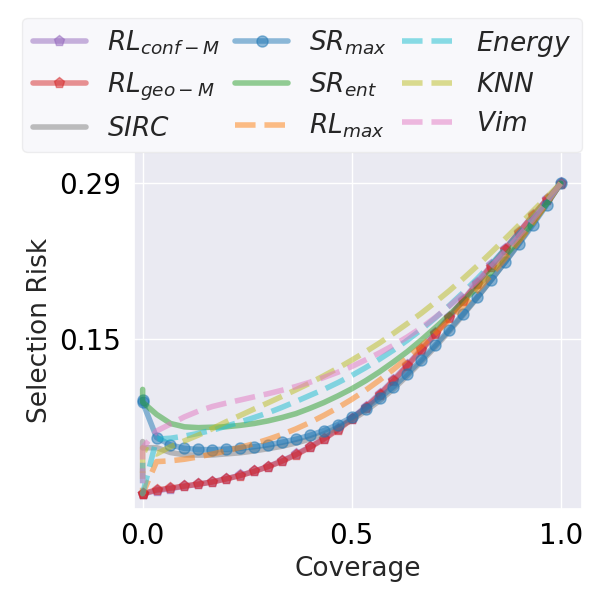}%
\\
&\multicolumn{4}{c}{\texttt{ImageNet} - VOLO}
\\
&\includegraphics[width=0.24\textwidth]{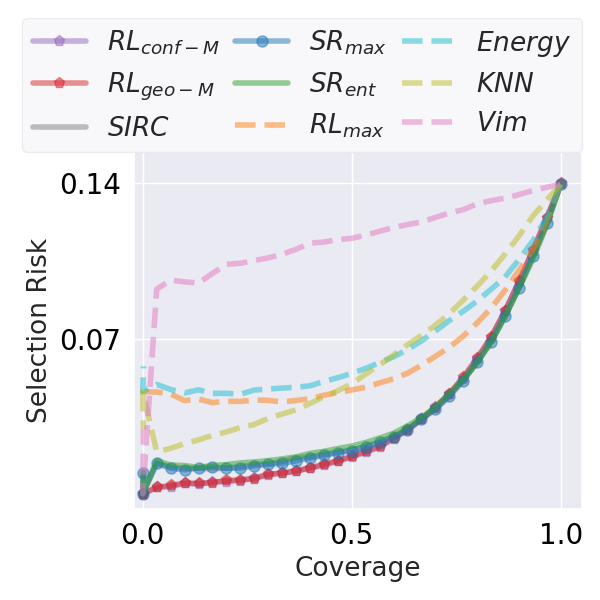}%
&\includegraphics[width=0.24\textwidth]{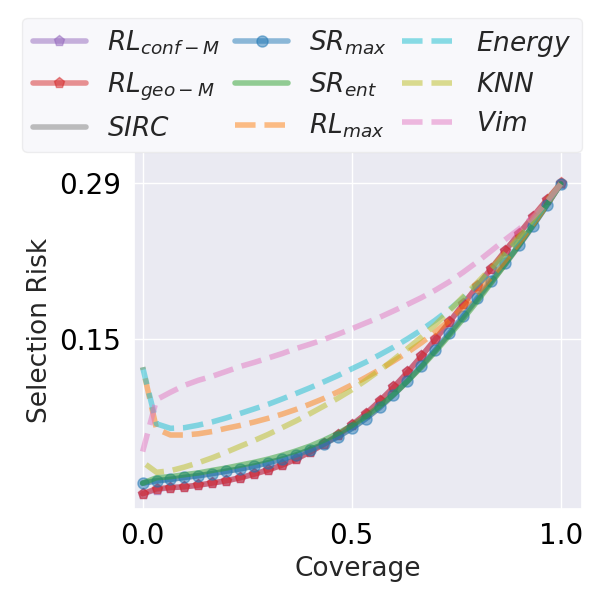}%
&\includegraphics[width=0.24\textwidth]{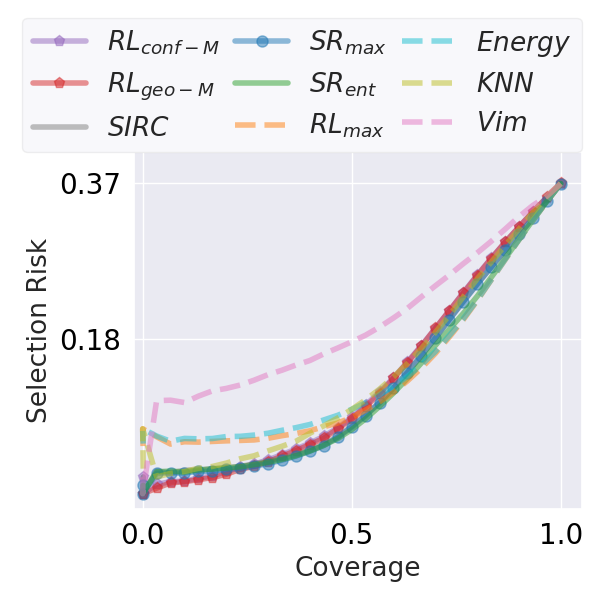}%
&\includegraphics[width=0.24\textwidth]{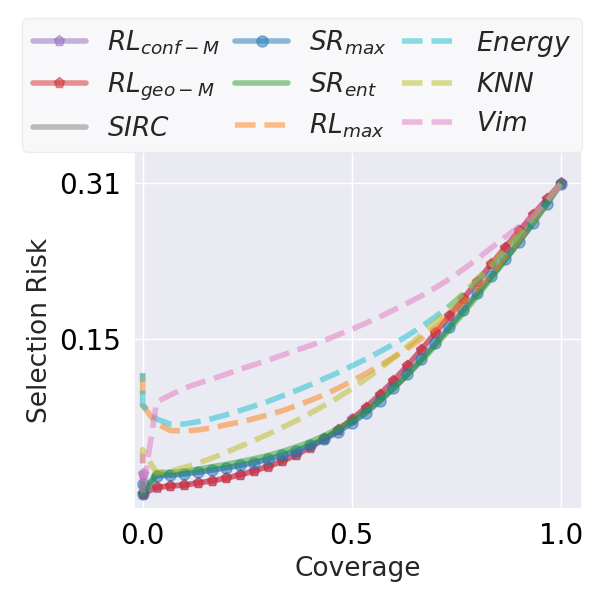}%
\\

&\small{In-D (ImageNet)}%
&\small{In-D + Shift (Cov)}%
&\small{In-D + Shift (Label)}%
&\small{In-D + Shift (both)}%
\end{tabular}}%
\caption{RC curves of different confidence-score functions on models \texttt{ConvNext}, \texttt{ResNext} and \texttt{VOLO} from \texttt{timm} for ImageNet. The four columns are RC curves evaluated using samples from In-D only, In-D and covariate-shifted only, In-D and label-shifted only, and all, respectively. We group the curves by whether they are originally proposed for SC (solid lines) or for OOD detection (dashed lines).}
\label{APPFig: ImageNet RC curves}
\vspace{-0.5em}
\end{figure}%
\begin{table*}[!tb]
\caption{Summary of AURC-$\alpha$ for \cref{APPFig: ImageNet RC curves}. The AURC numbers are \emph{on the $10^{-2}$ scale---the lower, the better}. The score functions proposed for SC are highlighted in gray, and the rest are originally for OOD detection. The best AURC numbers for each coverage level are highlighted in bold, and the $2^{nd}$ and $3^{rd}$ best scores are underlined.}
\label{app-tab: Partial AURC ImageNet}
\centering
\resizebox{1\linewidth}{!}{%
\begin{tabular}{l ccc c ccc c ccc c ccc}
\multicolumn{1}{c}{\small{\texttt{ImageNet} - ConvNext}} & \multicolumn{3}
{c}{\small{In-D}} & & \multicolumn{3}{c}{\small{In-D + Shift (Cov)}} & & \multicolumn{3}{c}{\small{In-D + Shift (Label)}} &
 & \multicolumn{3}{c}{\small{In-D + Shift (both)}}\\
\midrule
\multicolumn{1}{c}{\small{$\alpha$}} & 0.1 & 0.5 & 1 & & 0.1 & 0.5 & 1 & & 0.1 & 0.5 & 1 & & 0.1 & 0.5 & 1\\
\toprule
\rowcolor{Gray}
$RL_{\text{conf-M}}$ & \textbf{0.10} & \textbf{0.53} & \textbf{3.02} & & \textbf{0.26} & \underline{1.76} & \underline{8.20} & & \textbf{0.58} & \textbf{2.51} & 11.8 & & \textbf{0.34} & \underline{1.99} & \underline{8.88}\\
\rowcolor{Gray}
$RL_{\text{geo-M}}$ & \underline{0.15} & \underline{0.59} & \underline{3.10} & & \underline{0.31} & \textbf{1.75} & \textbf{8.14} & & \underline{0.75} & \underline{2.54} & 11.8 & & \underline{0.38} & \textbf{1.97} & \textbf{8.81}\\
\hdashline
\rowcolor{Gray}
SIRC & 1.96 & \underline{1.70} & \underline{3.59} & & 3.44 & \underline{3.23} & \underline{8.60} & & 5.94 & \underline{4.03} & \underline{11.5} & & 3.76 & \underline{3.46} & \underline{9.18}\\
\rowcolor{Gray}
$SR_{\text{max}}$ & 2.26 & 1.86 & 3.66 & & 3.73 & 3.40 & 8.70 & & 5.86 & 4.05 & \underline{11.4} & & 4.04 & 3.62 & 9.26\\

\rowcolor{Gray}
$SR_{\text{ent}}$ & 2.77 & 2.44 & 4.19 & & 4.78 & 4.33 & 9.54 & & 6.83 & 4.85 & 11.6 & & 5.13 & 4.56 & 10.1\\
\rowcolor{Gray}
$SR_{\text{doctor}}$ & 2.26 & 1.86 & 3.67 & & 3.73 & 3.41 & 8.74 & & 5.86 & 4.06 & \textbf{11.3} & &4.04 & 3.63 & 9.29\\

$RL_{\text{max}}$ & 5.43 & 4.77 & 5.81 & & 9.05 & 7.89 & 11.6 & & 10.5 & 7.73 & 13.2 & & 9.45 & 8.13 & 12.1\\
Energy & 6.66 & 6.70 & 7.54 & & 10.9 & 10.7 & 13.9 & & 11.9 & 9.78 & 14.6 & & 11.3 & 10.9 & 14.3\\
KNN & \underline{1.01} & 2.37 & 5.72 & & \underline{1.29} & 4.54 & 10.6 & & \underline{1.11} & 3.66 & 12.0 & & \underline{1.31} & 4.59 & 11.0\\
ViM & 15.1 & 9.84 & 9.49 & & 16.2 & 11.9 & 14.3 & & 14.1 & 9.57 & 14.5 & & 16.2 & 11.9 & 14.7\\
\toprule
\multicolumn{1}{c}{\small{\texttt{ImageNet} - ResNext}} & \\
\toprule
\rowcolor{Gray}
$RL_{\text{conf-M}}$ & \textbf{0.12} & \textbf{0.59} & \textbf{3.17} & & \textbf{0.29} & \underline{2.15} & \underline{9.38} & & \textbf{0.59} & \underline{3.22} & \underline{12.8} & & \textbf{0.38} & \underline{2.50} & \underline{10.2}\\
\rowcolor{Gray}
$RL_{\text{geo-M}}$ & \underline{0.17} & \underline{0.60} & \underline{3.18} & & \underline{0.34} & \textbf{2.14} & \textbf{9.33} & & \underline{0.65} & \textbf{3.16} & \underline{12.7} & & \underline{0.43} & \textbf{2.49} & \textbf{10.1}\\
\hdashline
\rowcolor{Gray}
SIRC & 1.71 & \underline{1.91} & \underline{3.94} & & 3.96 & \underline{4.18} & \underline{9.99} & & 7.77 & 5.88 & 13.1 & & \underline{4.47} & 4.57 & \underline{10.7}\\
\rowcolor{Gray}
$SR_{\text{max}}$ & 2.28 & 2.26 & 4.11 & & 4.88 & 4.69 & 10.3 & & 7.44 & 5.88 & 12.9 & & 5.36 & 5.06 & 11.0\\

\rowcolor{Gray}

$SR_{\text{ent}}$ & 3.38 & 3.42 & 5.37 & & 6.92 & 6.94 & 12.2 & & 9.46 & 7.70 & 13.9 & & 7.47 & 7.36 & 12.8\\
\rowcolor{Gray}
$SR_{\text{doctor}}$ & 2.29 & 2.28 & 4.17 & & 4.92 & 4.75 & 10.4 & & 7.47 & 5.92 & \underline{12.8} & & 5.39 & 5.12 & 11.1\\

$RL_{\text{max}}$ & \underline{1.57} & 2.34 & 4.79 & & \underline{2.98} & 4.82 & 10.9 & & \underline{2.37} & \underline{3.83} & \textbf{11.9} & & \underline{3.06} & 5.00 & 11.4\\
Energy & 3.08 & 3.90 & 6.17 & & 5.13 & 7.20 & 12.7 & & 3.68 & 5.34 & 13.2 & & 5.19 & 7.37 & 13.2\\
KNN & 3.23 & 4.84 & 7.61 & & 4.12 & 7.65 & 13.6 & & 3.40 & 5.85 & 13.5 & & 4.14 & 7.77 & 14.0\\
ViM & 4.68 & 6.13 & 7.79 & & 6.18 & 8.81 & 13.6 & & 5.09 & 6.82 & 13.6 & & 6.23 & 8.92 & 14.1\\
\toprule
\multicolumn{1}{c}{\small{\texttt{ImageNet} - VOLO}} & \\
\toprule
\rowcolor{Gray}
$RL_{\text{conf-M}}$ & \textbf{0.31} & \textbf{0.79} & \textbf{3.44} & & \textbf{0.46} & \underline{2.24} & 9.72 & & \underline{1.30} & \underline{3.79} & 13.3 & & \underline{0.68} & \underline{2.67} & \underline{10.6}\\
\rowcolor{Gray}
$RL_{\text{geo-M}}$ & \underline{0.37} & \underline{0.81} & \underline{3.46} & & \underline{0.50} & \textbf{2.23} & 9.73 & & \textbf{0.94} & \textbf{3.56} & 13.1 & & \textbf{0.66} & \textbf{2.64} & \underline{10.6}\\
\hdashline
\rowcolor{Gray}
SIRC & \underline{1.27} & 1.44 & 3.74 & & 1.35 & \underline{2.82} & \underline{9.56} & & 2.68 & 3.97 & \underline{12.9} & & 1.90 & 3.37 & \underline{
10.5}\\
\rowcolor{Gray}
$SR_{\text{max}}$ & 1.31 & \underline{1.42} & 3.72 & & \underline{1.33} & \underline{2.82} & \underline{9.59} & & 2.54 & \underline{3.78} & \underline{12.7} & & \underline{1.86} & \underline{3.36} & \underline{10.5}\\

\rowcolor{Gray}
$SR_{\text{ent}}$ & 1.47 & 1.59 & 3.83 & & 1.58 & 3.13 & 9.72 & & 2.71 & 3.87 & \textbf{12.4} & & 2.13 & 3.69 & \underline{10.6}\\
\rowcolor{Gray}
$SR_{\text{doctor}}$ & 1.31 & \underline{1.42} & \underline{3.71} & & \underline{1.33} & \underline{2.82} & \textbf{9.55} & & 2.54 & \underline{3.78} & \underline{12.7} & & \underline{1.86} & \underline{3.36} & \textbf{10.4}\\

$RL_{\text{max}}$ & 4.92 & 4.51 & 6.18 & & 6.32 & 7.13 & 12.5 & & 6.37 & 6.82 & 13.8 & & 7.07 & 7.84 & 13.4\\
Energy & 5.21 & 4.99 & 6.84 & & 6.88 & 8.24 & 13.5 & & 6.70 & 7.37 & 14.3 & & 7.62 & 8.95 & 14.4\\
KNN & 2.18 & 3.29 & 6.23 & & 2.10 & 5.03 & 11.7 & & \underline{2.27} & 4.85 & 13.7 & & 2.15 & 5.26 & 12.3\\
ViM & 9.38 & 10.7 & 11.9 & & 9.04 & 12.0 & 16.5 & & 10.4 & 13.5 & 21.1 & & 9.22 & 12.4 & 17.3\\
\bottomrule
\end{tabular}}
\end{table*}%

\section{Ablation experiments for the KNN score}
\label{App: KNN scores ablation}
We show in \cref{AppFig: ScNet RC curves} the SC performance of the KNN score on models \texttt{EVA}, \texttt{ConvNext}, \texttt{ResNext}, and \texttt{VOLO}, respectively, on \texttt{ImageNet} with all In-D and distribution-shifted samples. We can observe that (i) the SC performance of KNN is sensitive to the choice of hyperparameter $k$, and (ii) our selection $k=2$ achieves the best SC performance for KNN score on our \texttt{ImageNet} task. 
\begin{figure}[!htbp] 
\centering
\resizebox{1\columnwidth}{!}{%
\begin{tabular}{c c c c c}
\centering
&\small{\texttt{EVA}}
&\small{\texttt{ConvNext}}
&\small{\texttt{ResNext}}
&\small{\texttt{VOLO}}
\\
&\includegraphics[width=0.24\textwidth]{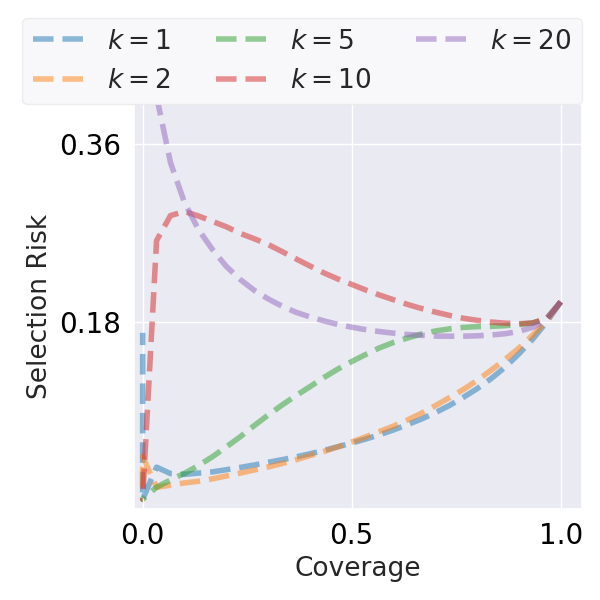}
&
\includegraphics[width=0.24\textwidth]{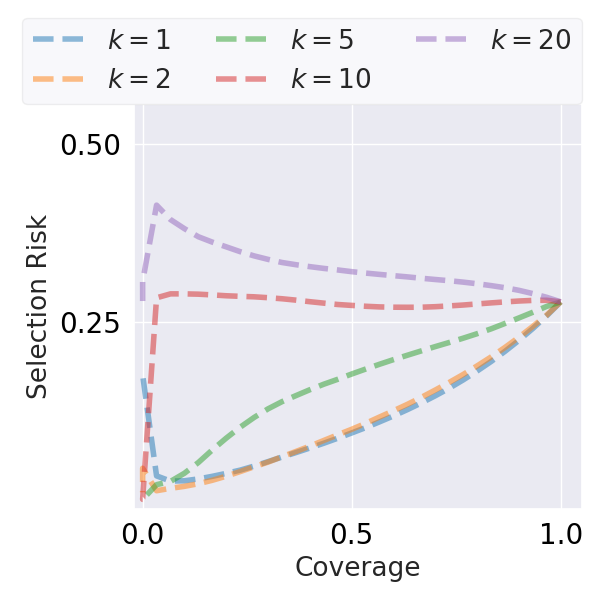}%
&\includegraphics[width=0.24\textwidth]{
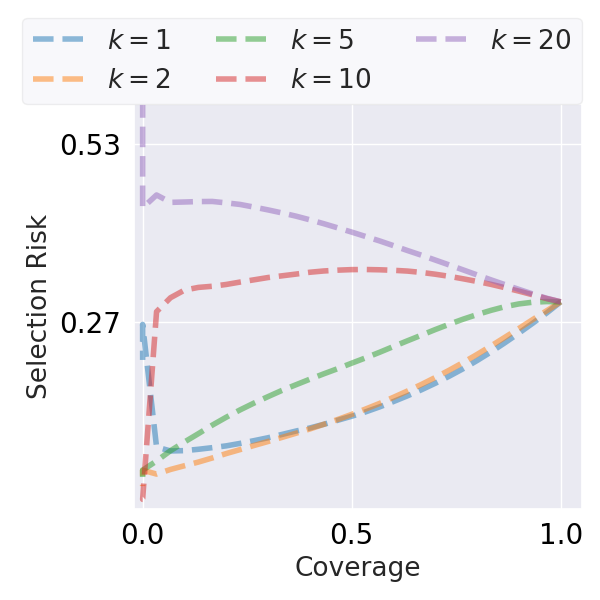}%
&\includegraphics[width=0.24\textwidth]{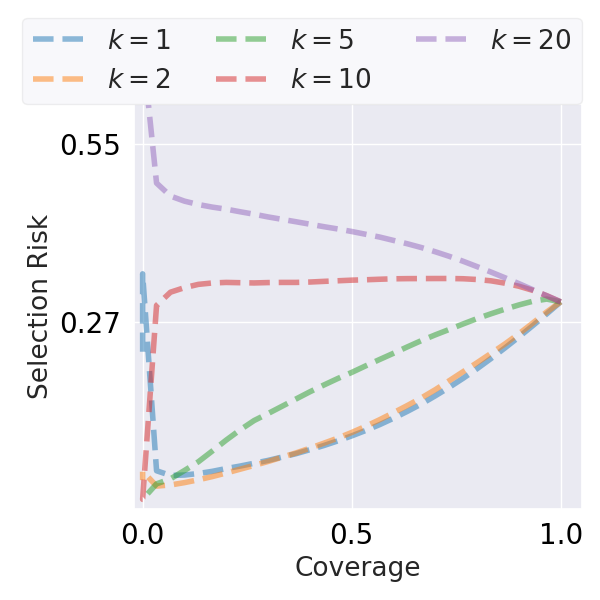}%

\end{tabular}}%
\caption{RC curves achieved by the KNN score with different $k$ on \texttt{ImageNet}}
\label{AppFig: ScNet RC curves}
\vspace{-0.5em}
\end{figure}%

\end{document}